\documentclass{article}

\usepackage{microtype}
\usepackage{graphicx}
\usepackage{subfigure}
\usepackage{booktabs} 
\usepackage{amsmath,amsthm,amssymb}
\usepackage{xcolor}

\usepackage{hyperref}

\DeclareMathOperator*{\E}{\mathbb{E}}

\newcommand{\gf}[1]{\textcolor{red}{GF: #1}}

\newcommand{\gap}{\normalfont \text{Gap}} 
\newcommand{\prox}{\normalfont \text{prox}} 
\newtheorem{definition}{Definition}
\newtheorem{theorem}{Theorem}
\newtheorem{lemma}{Lemma}
\newtheorem{remark}{Remark}

\usepackage[accepted]{icml2021}

\icmltitlerunning{Efficient Algorithms for Federated Saddle Point Optimization}

\begin{document}

\twocolumn[
\icmltitle{Efficient Algorithms for Federated Saddle Point Optimization}



\icmlsetsymbol{equal}{*}

\begin{icmlauthorlist}
\icmlauthor{Charlie Hou}{cmu}
\icmlauthor{Kiran K. Thekumparampil}{uiuc}
\icmlauthor{Giulia Fanti}{cmu}
\icmlauthor{Sewoong Oh}{uw}
\end{icmlauthorlist}

\icmlaffiliation{cmu}{Department of Electrical and Computer Engineering, Carnegie Mellon University, Pittsburgh, Pennsylvania, USA}
\icmlaffiliation{uiuc}{Department of Electrical and Computer Engineering, University of Illinois
at Urbana-Champaign, Urbana-Champaign, Illinois, USA}
\icmlaffiliation{uw}{Department of Computer Science and Engineering, University of Washington, Seattle, Washington, USA}

\icmlcorrespondingauthor{Charlie Hou}{charlieh@andrew.cmu.edu}
\icmlcorrespondingauthor{Kiran K. Thekumparampil}{thekump2@illinois.edu}
\icmlcorrespondingauthor{Giulia Fanti}{gfanti@andrew.cmu.edu}
\icmlcorrespondingauthor{Sewoong Oh}{sewoong@cs.washington.edu}
\icmlkeywords{Machine Learning, ICML, Optimization, Saddle-point, Federated Learning}

\vskip 0.3in
]



\printAffiliationsAndNotice{}  

\begin{abstract}
	We consider strongly convex-concave minimax problems in the federated setting, where the communication constraint is the main bottleneck. When clients are arbitrarily heterogeneous, a simple Minibatch Mirror-prox achieves the best performance. As the clients become more homogeneous, using multiple local gradient updates at the clients significantly improves upon Minibatch Mirror-prox by communicating less frequently. Our goal is to design an algorithm that can harness the benefit of similarity in the clients while recovering the Minibatch Mirror-prox performance under arbitrary heterogeneity (up to log factors). We give the first federated minimax optimization algorithm that achieves this goal. The main idea is to combine (i) SCAFFOLD (an algorithm that performs variance reduction across clients for convex optimization) to erase the worst-case dependency on heterogeneity and (ii) Catalyst (a framework for acceleration based on modifying the objective) to accelerate convergence without amplifying client drift. We prove that this algorithm achieves our goal, and include experiments to validate the theory.
%
\end{abstract}

\begin{table*}[t]
	
	\caption{Communication round complexities with strong convex-concavity.  For the convex optimization rates, $\mu$ denotes strong convexity and $\beta$ denotes smoothness.  (*) denotes a rate that assumes access to noiseless gradients once per communication round.  MD stands for Mirror Descent and MP stands for Mirror-prox.  A -S suffix means a gradient descent-ascent variant of the suffixed algorithm.  To compare rates using random synchronization to rates using deterministic synchronization, set $p = \frac{1}{\tau}$.  Let $z'$ be the final output of the algorithm and $z^*$ be the solution to the problem it is solving.}
	\renewcommand{\arraystretch}{1.4}
	\centering
	\begin{tabular}{ l l l }
		\hline 
		\textbf{Federated Algorithm} &  \textbf{Communication Complexity to reach $\E \|z' - z^*\|^2 < \epsilon$} \\ 
		\hline
		\textit{ Convex optimization} \\
		\hline
		\begin{tabular}{l} 
			FedAvg  {\small\cite{woodworth2020minibatch}}\\
			SCAFFOLD  {\small\cite{karimireddy2020scaffold}}\\
		\end{tabular}

		& \begin{tabular}{l}
			$\min\{\frac{\beta D}{\tau \mu \epsilon^{1/2}}, \frac{D^2}{\tau \mu \epsilon}\}+ \frac{\sigma^2 D^2}{n \tau \mu \epsilon^2} + \frac{\beta^{1/2} \sigma}{\tau^{1/2} \mu^{3/2} \epsilon^{1/2}} + \frac{\beta^{1/2} \zeta}{\mu^{3/2} \epsilon^{1/2}} $ \\
			$\frac{\sigma^2}{n \tau  \mu^2 \epsilon} + \frac{\beta}{\mu}$\\
		\end{tabular}
		\\\hline
		\textit{Convex-Concave Minimax optimization} \\
		\hline
		\begin{tabular}{l} 
			Minibatch MD {\small\cite{balamurugan2016stochastic}}\\
			Minibatch MP(*) {\small\cite{tseng1995linear}}\\
		\end{tabular}

		& \begin{tabular}{l}
			$ \frac{\sigma^2}{n \tau \mu^2 \epsilon} + \frac{\beta^2}{\mu^2}$ \\
			$\frac{\beta}{\mu} $ \\
		\end{tabular}
		\\\hline
		\begin{tabular}{l}
			FedAvg-S (Theorem~\ref{fedavgtheorem}) \\
			SCAFFOLD-S(*) (Theorem~\ref{scaffoldtheorem})\\
			SCAFFOLD-Catalyst-S(*) (Theorem~\ref{catalysttheorem})\\
		\end{tabular}

		&
		\begin{tabular}{l}
			$\frac{p \beta^2}{\mu^2} + \frac{p \sigma^2}{n \mu^2 \epsilon}  + \frac{p^{1/2} \beta \sigma}{\mu^{2}\epsilon^{1/2} }  + \frac{\beta \zeta}{ \mu^{2} \epsilon^{1/2} }$ \\
			$ \min\{\frac{\beta^2}{\mu^2},  \frac{p \beta^2}{\mu^2} + \frac{p \sigma^2}{n \mu^2 \epsilon}  + \frac{p^{1/2}\beta \sigma}{ \mu^{2} \epsilon^{1/2}} + \frac{\beta \zeta}{ \mu^{2} \epsilon^{1/2} }\}$\\
			$\min\{\frac{\beta}{\mu} , \frac{p \beta^2}{\mu^2} + \frac{p \sigma^2}{n \mu^2 \epsilon} + \frac{p^{1/2}\beta \sigma}{ \mu^{2} \epsilon^{1/2}}+ \frac{\beta \zeta}{ \mu^{2} \epsilon^{1/2} }\}  $\\
		\end{tabular}
		\\\hline 
	\end{tabular}
	\label{complexitytable}
\end{table*}
\begin{table}[t]
    \centering
    \caption{List of notations}
	\renewcommand{\arraystretch}{1.2}
	\begin{tabular}{ l l }
	\hline
	Symbol & Meaning \\
	\hline
	\begin{tabular}{ l }
		$\mu$ \\
        $\beta$ \\
        $\zeta$ \\
        $p$ \\
        $n$\\
        $\kappa$ \\
        $D$ \\
        $\tau$ \\
        $\sigma^2$\\
	\end{tabular} &
	\begin{tabular}{ l }
		Strong convex-(concavity)\\
        Smoothness \\
        Client heterogeneity \\
        Probability of synchronization \\
        Number of clients \\
        Condition number $\frac{\beta}{\mu}$ \\
        Upper bound on norm of optimal point \\
        Number of local steps per round \\
        Gradient query variance
	\end{tabular}
\end{tabular}
\end{table}
\section{Introduction}
\label{intro}
In federated learning \cite{mcmahan2017communication}, a set of distributed clients or devices interact with a central server to learn one or more models without directly sharing any party's data with the central server. 
For many real-world applications, the bottleneck in federated learning is communication \cite{karimireddy2020scaffold}, since large models such as deep neural networks can be expensive to transmit over slow or unreliable communication channels.  
Therefore, a central goal in federated learning is to use client (local) computation effectively to reduce the communication complexity of learning, even in the presence of heterogeneous clients with data drawn from different distributions.  

In this work, we study this problem for federated saddle point optimization problems over $n$ clients of the form 
\begin{align}
    \label{eq:objective}
    \min_{x \in \mathbb{R}^m} \max_{y \in \mathbb{R}^d} \left\{ f(x,y) := \frac{1}{n}\sum_{i=1}^n f_i(x,y) \right\}
\end{align}
where the functions $f_i: \mathbb{R}^m \times \mathbb{R}^d \to \mathbb{R}$ are strongly convex-concave.
In particular, we consider the heterogeneous setting where the $i$th client has access to data from distribution $\mathcal D_i$, and the $f_i$ of each client are of the form 
\begin{align}
\label{eq:losses}
f_i(x, y) = \mathbb{E}_{\xi \sim \mathcal D_i} [f_i(x,y,\xi)].
\end{align}

Federated minimax problems arise in many natural settings.  For instance, (not necessarily convex-concave) GANs \cite{goodfellow2014generative} have been applied in the federated setting \cite{augenstein2019generative}, where the goal is to train a model that will mimic data distributed among clients.  For (strongly) convex-concave examples, formulation (\ref{eq:objective}) can be used to represent federated robust optimization \cite{ben2009robust} and primal-dual optimization of federated supervised learning problems \cite{balamurugan2016stochastic}.

To date, research in federated learning has focused on federated \textit{minimization} problems.  The most well-known algorithm developed to solve federated minimization is FedAvg \cite{mcmahan2017communication}.  FedAvg is known to perform well when data is homogeneous or near-homogeneous across clients, but possibly performs badly when data is arbitrarily heterogeneous across clients \cite{woodworth2020minibatch}.  To remedy this, SCAFFOLD \cite{karimireddy2020scaffold} proposes variance reduction across clients to prevent client heterogeneity from hurting convergence.

However, the analyses of these algorithms, and other federated minimization algorithms, do not apply to the federated minimax setting.  All past analyses rely on a particular gradient co-coercivity property that is true for function values but not true for the primal-dual gap (the metric used to measure the quality of a solution to (\ref{eq:objective})).  As a result, there currently exist no algorithms with proven convergence guarantees for (\ref{eq:objective}).

One straightforward approach is to use FedAvg/SCAFFOLD to solve (\ref{eq:objective}), except replacing gradient descent with gradient ascent-descent.  We will refer to the gradient ascent-descent variants of these algorithms also as FedAvg-S and SCAFFOLD-S, respectively (``S'' for saddle-point).  However, for these algorithms we recover inferior worst-case communication complexity guarantees compared to Minibatch Mirror-prox \cite{nemirovski2004prox}.  Specifically, FedAvg-S continues to be slow the presence of large client heterogeneity, while the communication complexity of SCAFFOLD-S is worse than that of Minibatch Mirror-prox by a condition number in the strongly convex case.  This is perhaps not surprising, as FedAvg-S and SCAFFOLD-S are not accelerated algorithms, while mirror-prox can be viewed as an accelerated algorithm.  Naively, one might try to add acceleration to the local gradients.  However, this can lead to the amplification of \emph{client drift}, which describes the dispersion of client iterates away from each other.  In the heterogeneous client setting, client drift is unavoidable when taking local steps.  As a result, provable improvement with client-level acceleration has been limited to the homogeneous data case \cite{yuan2020federatedac}.  

Our solution is to apply the catalyst framework for acceleration \cite{lin2015universal} at the central server level to SCAFFOLD-S.  Intuitively, we ask SCAFFOLD-S to solve a series of regularized versions of the original objective.  If the regularization is large, the number of modified objectives to be solved is large, but the number of communication rounds required to solve each modified objective is small (and vice versa).  Because the regularization is applied uniformly across clients, client drift is not amplified.  By strategically balancing the number of iterations to solve the outer problem versus the inner problem, we can accelerate SCAFFOLD-S to achieve the same worst-case communication guarantee of Minibatch Mirror-prox while also providing an advantage over Minibatch Mirror-prox when the client data is similar.  
\subsection{Our Contributions}
\begin{itemize}
    \item We develop a novel analysis of federated optimization that can apply to federated minimax problems, which we use to derive convergence rates for SCAFFOLD-S and FedAvg-S in the minimax setting.  We also show that SCAFFOLD-S can take advantage of local computation to reduce communication complexity in the strongly convex(-concave) case, past only quadratics as \cite{karimireddy2020scaffold} had shown.
    \item We use a novel application of the catalyst framework on SCAFFOLD-S to develop a new algorithm, SCAFFOLD-Catalyst-S, which we prove achieves the same worst-case communication complexity as Minibatch Mirror-prox.  
    \item We also prove that our algorithm can take advantage of similarity in the client objectives, showing that local computation can reduce communication complexity under favorable conditions.
    \item Our application of the catalyst framework is the first acceleration of a federated algorithm, minimax or otherwise, in the presence of heterogeneous clients.  We conjecture this approach to the acceleration of federated optimization will generalize past the minimax setting to the minimization setting.    
\end{itemize}

\section{Related Work}
Developing and analyzing algorithms for federated learning has been an active area of research.  FedAvg's minimization convergence properties in the homogeneous client setting was first established by \cite{stich2018local}, and was tightened later by \cite{woodworth2020local} with accompanying lower bounds.  Later, \cite{khaled2020tighter} proved convergence rates for FedAvg in the heterogeneous client setting, which was later tightened by \cite{karimireddy2020scaffold} and then also by \cite{woodworth2020minibatch}.  A lower bound was also established by both \cite{karimireddy2020scaffold} and \cite{woodworth2020minibatch} showing that FedAvg's convergence rate necessarily scales with client heterogeneity.  

As a result, new algorithms based on variance reduction across clients were proposed to remove the convergence rate dependence on heterogeneity \cite{liang2019variance} \cite{karimireddy2020scaffold}, analogous to how standard variance reduction \cite{johnson2013accelerating} can remove the convergence rate dependence on gradient query variance.  SCAFFOLD, the most well known client variance-reduced algorithm, \cite{karimireddy2020scaffold} was able to match the worst-case rate of minibatch SGD (but notably, not accelerated minibatch SGD) under arbitrary client heterogeneity.  A lower bound on a broad class of federated minimization algorithms from \cite{woodworth2020minibatch} suggests that accelerated minibatch SGD's rate under arbitrary client heterogeneity is optimal, though notably this lower bound does not directly apply to SCAFFOLD.  Finally, \cite{gorbunov2020local} unifies the analysis of all the above algorithms, proving upper bounds for each.

Variance reduction has been an important research direction in minimax optimization as well, starting with the work of \cite{balamurugan2016stochastic}, which provided a framework for analyzing variance reduction in the minimax setting.  Furthermore, they established (up to log factors) the current state of the art rates by combining variance reduction together with the catalyst framework \cite{lin2015universal}.  Later \cite{carmon2019variance} combined \cite{nemirovski2004prox}'s prox-method with variance reduction to improve the convergence rate for matrix games.   

In light of some evidence that federated minimization algorithms have a lower bound equal to the accelerated minibatch SGD rate, our goal is to develop a federated minimax optimization algorithm that has the same worst-case communication complexity as Minibatch Mirror-prox \cite{nemirovski2004prox} (mirror-prox can be seen as the AGD of minimax optimization), while maintaining an advantage over Minibatch Mirror-prox when client data is similar.  We achieve this goal by combining techniques from the minimax variance reduction literature and the federated minimization literature cited above.

\section{Preliminaries}

We start with the following standard definitions for strongly-convex-concave minimax problems.
\begin{definition}
    $g$ is $\mu$-strongly convex-concave, with $\mu \geq 0$, if for any $x, x', y, y'$,
    \begin{align*}
        g(x',y) &\geq g(x,y) + \langle \nabla_x g(x, y), x' - x \rangle \\
        & \ \ \ + \frac{\mu}{2} \|x' - x\|^2 \\
        -g(x,y') &\geq -g(x,y) + \langle -\nabla_y g(x, y), y' - y \rangle \\
        & \ \ \ + \frac{\mu}{2} \|y' - y\|^2.
    \end{align*}
\end{definition}

\begin{definition}
    $g$ is $\beta$-smooth if for any $x, x', y, y'$,
    \begin{align*}
        \| \nabla_x g(x,y) - \nabla_x g(x',y)\| &\leq \beta \|x - x'\| \\
        \| \nabla_y g(x,y) - \nabla_y g(x,y')\| &\leq \beta \|y - y'\| \\
        \| \nabla_x g(x,y) - \nabla_x g(x,y')\| &\leq \beta \|y - y'\| \\
        \| \nabla_y g(x,y) - \nabla_y g(x',y)\| &\leq \beta \|x - x'\|.
    \end{align*}
\end{definition}
We let $\kappa := \frac{\beta}{\mu}$, and use the notation $z=(x,y)$, $z \in \mathbb{R}^m \times \mathbb{R}^d$ to refer to the concatenation of $x$ and $y$.
\begin{definition}
    The gradient mapping $G$ is defined as 
    \begin{align}
        G(z) = (\nabla_x f(z), -\nabla_y f(z)).
    \end{align}
    Similarly,
    \begin{align}
        G_i(z) = (\nabla_x f_i(z), -\nabla_y f_i(z)).
    \end{align}
    Let $\hat{G_i}$ be an unbiased estimate of $G_i(z)$ with variance $\sigma^2$.
\end{definition}
\begin{definition}
    A point $z^* = (x^*, y^*)$ is considered a minimax-optimal point for $g$ if for all $x \in \mathbb R^m$, $y\in \mathbb R^d$,
    \begin{align}
        g(x^*, y) \leq g(x^*, y^*) \leq g(x, y^*).
    \end{align}
\end{definition}
From now on, we will let $z^*=(x^*,y^*)$ denote the minimax-optimal point for $f$.  

Finally, we introduce a notion of solution quality for minimax optimization:
\begin{definition}
    The duality gap of point $z$ with respect to $z^*$ is defined as
    \begin{align}
        \text{\normalfont Gap}^*(z) = f(x, y^*) - f(x^*, y).
    \end{align}
\end{definition}
Next, we introduce a common notion of client heterogeneity (modified for minimax optimization), which was introduced in \cite{woodworth2020minibatch}.
\begin{definition}
    A set of functions $\{ f_i \}_{0 \leq i \leq n}$ are called $\zeta$-heterogeneous if for all $i,j \in [n]$ and $z \in \mathbb{R}^m \times \mathbb{R}^d$,
    \begin{align}
        \|G_i(z) - G_j(z)\|^2 \leq \zeta^2.
    \end{align}
\end{definition}
\begin{algorithm}[tb]
    \caption{Framework (\ref{eq:localframework})}
    \label{alg:framework}
 \begin{algorithmic}
	\STATE {\bfseries Server Input:} initial $z^0$, stepsizes $\gamma_l, \gamma_g$
	\STATE probability of communication $p$
    \STATE {\bfseries Client Input:} local function $f_i$
	\STATE {\bfseries set} $\tilde{z}^0 = z^0$
    \FOR{$k=0, 1, \dots$}
		\STATE Flip a coin $c_k$; $c_k = 1$ w.p. $p$, and $c_k = 0$ otherwise
        \FOR{each client $i$ in parallel}
            \STATE $z_i^{k+1} \gets z_i^{k} - \gamma_l g_i^k$
            \IF{$c_k = 1$}
                \STATE Clients {\bfseries communicate} $\sum_{l = k'}^k g_i^l$ to server
                \STATE Server {\bfseries broadcasts} $\frac{1}{n}\sum_{i=1}^n \sum_{l = k'}^k g_i^l$ 
			    \STATE $\tilde{z}^{k+1} \gets \tilde{z}^k - \gamma_g \frac{1}{n} \sum_{i=1}^n \sum_{l = k'}^k g_i^l$
                \STATE $z_i^{k+1} \gets \tilde{z}^{k+1}$
                \STATE Server {\bfseries broadcasts} $\tilde{z}^{k+1}$
                \STATE {\bfseries More communication} depending on algorithm 
		    \ENDIF
            \
        \ENDFOR
	\ENDFOR
 \end{algorithmic}
 \end{algorithm}

Throughout this work, we consider optimization algorithms that fit the update framework proposed in \cite{gorbunov2020local}:
\begin{align}
	\label{eq:localframework}
	z_i^{k+1} \gets
	\begin{cases}
		z_i^{k'} - \gamma_g \frac{1}{n} \sum_{i=1}^n \sum_{l = k'}^{k} g_i^l  &  $w.p. $ p \\
		z_i^k - \gamma_l g_i^k & $w.p  $ 1-p
	\end{cases}
\end{align}
Here $z_i^{k}$ is client $i$'s iterate after $k$ iterations, $g_i^k$ is the local first-order direction taken, $\gamma_l$ is the local stepsize, $\gamma_g$ is the global stepsize, $k'$ is the last iteration on which synchronication occurred, and $p$ is the probability of synchronizing at any given $k$.  We provide a more intuitive way of expressing (\ref{eq:localframework}) in Algorithm~\ref{alg:framework}, which explicitly shows what is communicated and when.

Note that most prior federated learning literature has focused on deterministic synchronization, where time is divided into \emph{rounds}, and each client takes a fixed number $\tau>0$ of local update steps per communication round before synchronizing globally.  
We are instead (a) letting $\tau$ be random, and (b) not explicitly counting the number of rounds (i.e., global synchronization updates). 
For comparison with prior work, we set $p = \frac{1}{\tau}$, so the expected number of communication rounds executed at iteration $k$ can be computed as $kp$.

\section{Baselines}
\label{sec:baselines}
In this section, we define and cover the baseline algorithms: Minibatch Mirror Descent, Minibatch Mirror-prox, FedAvg-S, SCAFFOLD-S, and present their convergence rates in the federated minimax optimization setting. 

\subsection{Minibatch Mirror Descent}
We present Minibatch Mirror Descent in the same way \cite{woodworth2020minibatch} presents Minibatch SGD.  During a communication round, each client takes $\tau$ (here we let $\tau$ be deterministic) stochastic gradient mappings at the same point.  These gradient mappings are then averaged within the client and then sent to the central server.  The server collects all $\tau n$ gradient mappings, averages them, and takes a step in the resulting direction.  In the language of framework (\ref{eq:localframework}), we let $g_i^k = \hat{G}(z_i^k)$, $\gamma_l = 0$, and $\gamma_g$ be the stepsize.  Because we are simply performing a mirror descent step with minibatch size $\tau n$ at each step, the communication complexity is $\tilde{\mathcal{O}}(\frac{\beta^2}{\mu^2} + \frac{\sigma^2}{\tau n \mu \epsilon}$) in the strongly convex-concave case \cite{balamurugan2016stochastic}.

\subsection{Minibatch Mirror-prox}
After Minibatch Mirror Descent, it is natural to consider its accelerated counterpart.  This algorithm doesn't fit neatly into the notation we have defined so far, so we define new notation.  Let $\hat{z}^r$ be the iterate held at the server at the $r$-th communication round.  Minibatch Mirror-prox updates as follows: first, the central server collects $\tau n$ gradient mappings evaluated at $\hat{z}^r$, producing $\hat{z}^{r + 1/2}$.  Then the algorithm collects $\tau n$ gradient mappings evaluated at $\hat{z}^{r + 1/2}$, producing $\bar{g}^{r+1/2}$.  Finally, the server updates as $\hat{z}^{r + 1} \gets \hat{z}^{r} - \eta \bar{g}^{r+1/2}$.  So in the language of framework (\ref{eq:localframework}), $g_i^k = G_i(\hat{z}^{r + 1/2})$ (that is, evaluating a full batch gradient), and $\eta_l = 0$, where $\hat{z}^{r + 1/2}$ was defined earlier.  Because we are just taking a mirror-prox step per two communication rounds, the communication complexity in the strongly convex-concave case is $\tilde{\mathcal{O}}(\frac{\beta}{\mu})$ \cite{tseng1995linear} (assuming $\sigma = 0$).

\subsection{FedAvg-S}
The FedAvg-S algorithm follows framework (\ref{eq:localframework}) by
taking $g_i^k = \hat{G}_i(z_i^k)$ and $\gamma_g= \gamma_l$.  We will now state our convergence result for this algorithm.
\begin{theorem}
	\label{fedavgtheorem}
	For $\beta$-smooth and $\mu$-strongly convex-concave functions $\{f_i\}$ ($\mu > 0)$, the output of FedAvg-S, $z'$, has $\E \text{\normalfont Gap}^*(z') \leq \epsilon$ after 
	\begin{align}
		\tilde{\mathcal{O}}(\frac{p \beta^2}{\mu^2} + \frac{p \sigma^2}{n \mu \epsilon}  + \frac{p^{1/2} \beta \sigma}{\mu^{3/2}\epsilon^{1/2} } + \frac{\beta \zeta}{ \mu^{3/2} \epsilon^{1/2} })
		\label{eq:fedavg-value}
   	\end{align}
	communication rounds, 
	and $\E \|z' - z^*\|^2 \leq \epsilon$ after 
	\begin{align}
		\tilde{\mathcal{O}}(\frac{p \beta^2}{\mu^2} + \frac{p \sigma^2}{n \mu^2 \epsilon}  + \frac{p^{1/2} \beta \sigma}{\mu^{2}\epsilon^{1/2} }  + \frac{\beta \zeta}{ \mu^{2} \epsilon^{1/2} })
	\end{align}
	communication rounds given an appropriate choice of stepsize, where $\tilde{\mathcal{O}}$ hides both logarithmic and constant factors. 
\end{theorem}
The most natural comparison to make is between Minibatch Mirror Descent and FedAvg-S, as they are both unaccelerated, federated algorithms for minimax (similar to how FedAvg and Minibatch SGD are unaccelerated, federated algorithms for minimization).  In the case when $\sigma = 0, \zeta = 0$, FedAvg-S gets a speedup of $p$ in the first term vs Minibatch Mirror Descent: $\frac{\beta^2}{\mu^2} \to \frac{p \beta^2}{\mu^2}$.  When $\sigma = 0$ but $\zeta > 0$, FedAvg-S incurs an extra $\frac{\beta^{1/2} \zeta}{\mu^{3/2} \epsilon^{1/2}}$ term.  When $\zeta$ is small, FedAvg-S can be faster than batch mirror descent.  Otherwise, Minibatch Mirror Descent will outperform FedAvg-S.  When $\sigma > 0$, the noise terms $\frac{\sigma^2}{n \tau \mu^2 \epsilon}$ and $\frac{p \sigma^2}{n \mu^2 \epsilon}$ match; however, FedAvg-S incurs an extra term $\frac{p^{1/2} \beta \sigma}{\mu^{2}\epsilon^{1/2} }$ as variance in gradient queries also increase client drift.

We can compare this rate with the rate achieved by FedAvg for minimization using Table~\ref{complexitytable}.  In this case, the interesting comparisons are between the terms incurred by client drift, which are the last two terms for both rates in Table~\ref{complexitytable}.  Observe that in both terms, FedAvg-S loses a factor of $\sqrt{\kappa}$.  This is not surprising, as we will discuss later in Remark~\ref{losingkappa}.

\subsection{SCAFFOLD-S}
\begin{algorithm}[tb]
    \caption{SCAFFOLD-S($\{f_i\}$, p)}
    \label{alg:scaffold}
 \begin{algorithmic}
	\STATE {\bfseries Server Input:} initial $z^0$, probability of communication $p$
	\STATE {\bfseries Server Input:}  Stepsizes $\gamma_g, \gamma_l$
    \STATE {\bfseries Client Input:} local function $f_i$
	\STATE {\bfseries set} $\tilde{z}^0 = z^0$
    \FOR{$k=0, 1, \dots$}
		\STATE Flip a coin $c_k$; $c_k = 1$ w.p. $p$, and $c_k = 0$ otherwise
        \FOR{each client $i$ in parallel}
			\STATE $g_i^k \gets \hat{G_i}(z_i^k) - \hat{G_i}(\tilde{z}^k) + G(\tilde{z}^k)$
            \STATE $z_i^{k+1} \gets z_i^{k} - \gamma_l g_i^k$
            \IF{$c_k = 1$}
				\STATE Clients {\bfseries communicate} $\sum_{l = k'}^k g_i^l$ to server
				\STATE Server {\bfseries broadcasts} $\frac{1}{n}\sum_{i=1}^n \sum_{l = k'}^k g_i^l$ 
				\STATE $\tilde{z}^{k+1} \gets \tilde{z}^k - \gamma_g \frac{1}{n} \sum_{i=1}^n \sum_{l = k'}^k g_i^l$
				\STATE $z_i^{k+1} \gets \tilde{z}^{k+1}$
                \STATE Server {\bfseries broadcasts} $\tilde{z}^{k+1}$ 
                \STATE Clients {\bfseries communicate} $G_i(\tilde{z}^{k+1})$ to server
				\STATE Server {\bfseries broadcasts} $G(\tilde{z}^{k+1})$
		    \ENDIF
            \
        \ENDFOR
	\ENDFOR
 \end{algorithmic}
 \end{algorithm}
The SCAFFOLD-S algorithm is outlined in Algorithm \ref{alg:scaffold}.
Here, we take $g_i^k = \hat{G_i}(z_i^k) - \hat{G_i}(\tilde{z}^k) + G(\tilde{z}^k)$, where $\tilde{z}^k$ is the last synchronized iterate before or at iteration $k$.
Notice that this requires the full calculation of $G(\tilde{z}^k)$ at each synchronization event, which increases the communication cost per communication round by a constant factor.  
Also note that Algorithm \ref{alg:scaffold} is a simplified version of the SCAFFOLD algorithm in \cite{karimireddy2020scaffold}, though the simplified Algorithm \ref{alg:scaffold} is used for some of the analysis in \cite{karimireddy2020scaffold}, and the convergence properties of both variants are similar.

\begin{theorem}
	\label{scaffoldtheorem}
	Given that $\{f_i\}$ are all $\beta$-smooth and $\mu$-strongly convex for $\mu > 0$, the output of SCAFFOLD-S, $z'$, has $\E \text{\normalfont Gap}^*(z') \leq \epsilon$ in  
	\begin{align}
		\tilde{\mathcal{O}}(\frac{\beta^2}{\mu^2})
	\end{align}
	communication rounds, and, by applying strong convex-concavity of $f$, has $\E \|z' - z^*\|^2 \leq \epsilon$ in 
	\begin{align}
		\tilde{\mathcal{O}}(\frac{\beta^2}{\mu^2})
	\end{align}
	communication rounds with an appropriate choice of stepsizes. Furthermore, the same bounds that hold for FedAvg-S in Theorem \ref{fedavgtheorem} also hold for SCAFFOLD-S.
\end{theorem}
First, we compare SCAFFOLD-S to FedAvg-S.  FedAvg-S's upper bound on communication complexity applies to SCAFFOLD-S as well by Theorem~\ref{scaffoldtheorem}, so SCAFFOLD does at least as well as FedAvg-S.  In the case where $\zeta$ is of greater order than $\beta \epsilon^{1/2}$, then the SCAFFOLD federated minimax optimization rate outperforms FedAvg-S.  

Next, we can compare SCAFFOLD-S's federated minimax rate to Minibatch Mirror Descent's rate, as they are both unaccelerated algorithms for federated minimax optimization.  As FedAvg-S's rates apply to SCAFFOLD-S as well, the comparisons between FedAvg and batch mirror descent apply as well when $\zeta$ and $\sigma$ are small; that is, SCAFFOLD-S has an advantage.  On the other hand, when $\zeta$ is large, SCAFFOLD-S's rate matches that of Minibatch Mirror Descent.

Finally, we can compare SCAFFOLD-S with SCAFFOLD.  In the case where $\zeta$ is of order larger than $\beta \epsilon^{1/2}$ and $\sigma = 0$, then the rate for SCAFFOLD-S is worse than the rate for SCAFFOLD by a factor of $\kappa$.  This is expected, since in this regime SCAFFOLD is approximately performing a Minibatch SGD step \cite{woodworth2020minibatch} on the global objective per communication round, and SCAFFOLD-S in the federated minimax setting is approximately performing a Minibatch Mirror Descent step on the global objective per communication round.  Now allowing $\sigma > 0$, we see that the noise terms match as well.

SCAFFOLD's rates are missing a ``min'' expression compared to our SCAFFOLD-S rate.  This is because \cite{karimireddy2020scaffold} did not prove SCAFFOLD's ability to take advantage of local steps past the quadratic case.  We conjecture that a similar ``min'' expression should also be true for SCAFFOLD's federated minimization rate.

Altogether, we have demonstrated federated minimax optimization rates for FedAvg-S and SCAFFOLD-S.  In particular, we have shown that SCAFFOLD-S under arbitrary client heterogeneity matches the Minibatch Mirror Descent rate.  However, SCAFFOLD-S cannot achieve Minibatch Mirror-prox's communication complexity.  This is expected, because SCAFFOLD-S does not employ acceleration, while Minibatch Mirror-prox is an accelerated algorithm.  In what follows, we show how to accelerate SCAFFOLD-S to obtain the same  worst-case communication complexity as Minibatch Mirror-prox under arbitrary client heterogeneity, while still being able to perform better than Minibatch Mirror-prox if client data is similar.

\begin{remark}
	\label{losingkappa}
	\normalfont 
	In the centralized (non-federated) setting, upper bounds for the minimax setting are often larger by a factor of $\kappa$ compared to their counterparts in the minimization setting.  For instance, mirror descent achieves an iteration complexity of $\tilde{O}(\kappa^2)$ for minimax optimization, while its counterpart for minimization, gradient descent, achieves a complexity of $\tilde{O}(\kappa)$ for minimization.  Similarly, for accelerated algorithms, mirror-prox has an iteration complexity of $\tilde{O}(\kappa)$ and Nesterov AGD has an iteration complexity of $\tilde{O}(\sqrt{\kappa})$ \cite{bubeck2014convex}.   

	Improvements of $\kappa$ or $\sqrt{\kappa}$ in the convergence rates of (centralized) minimization problems compared to minimax problems can be derived from the \emph{gradient co-coercivity property} for $G(z) = \nabla_z f(z)$, i.e. the property that for all $z,z'$,
	\begin{align}
		\|G(z) - G(z')\|^2 \leq \beta (G(z) - G(z'))^T(z - z').
	\end{align}  
	However, this property is not true in the minimax setting for $G(x,y) = (\nabla_x f(x,y), -\nabla_y f(x,y))$, where $z = (x,y)$.
	Furthermore, this observed gap between minimization and minimax was verified by a lower bound \cite{zhang2019lower}.
	
	In most analysis of first-order federated minimization algorithms (e.g., FedAvg, SCAFFOLD, or the unified analysis of \cite{gorbunov2020local}), gradient co-coercivity of $\nabla_z f(z)$ is used to bound the error that arises from not synchronizing iterates in every round.  
	Hence, 
	we find the lack of gradient co-coercivity for minimax problems causes a loss of $\kappa$ factors in the federated setting as well.  We lose a factor of $\kappa$ simply to the gap mentioned earlier (as we are working with unaccelerated algorithms); we also lose a factor of $\sqrt{\kappa}$ in the term accounting for client drift (if present).  
\end{remark}

\section{SCAFFOLD-Catalyst-S}
\begin{algorithm}[tb]
    \caption{SCAFFOLD-Catalyst-S}
    \label{alg:catalyst}
 \begin{algorithmic}
    \STATE {\bfseries Server Input:} regularization $\theta$, initial meta-iterate $\bar{z}^0$,
    \STATE probability of communication $p \in (0,1]$
    \STATE {\bfseries Client Input:} local function $f_i$ 
    \FOR{$t=0, 1, \dots$}
        \STATE {\bfseries communicate} $\bar{z}^t$ to all clients 
        \FOR{each client $i$ in parallel}
            \STATE {\bfseries set} $f_i^{\theta}(z, \xi) = f_i(z, \xi) + \frac{\theta}{2}\|x - \bar{x}^t\|^2 - \frac{\theta}{2}\|y - \bar{y}^t\|^2$
        \ENDFOR
        \STATE $\bar{z}^{t+1} \gets \text{SCAFFOLD-S}(\{f_i^{\theta}\}, p)$
    \ENDFOR
 \end{algorithmic}
 \end{algorithm}
We utilize the catalyst framework for acceleration \cite{lin2015universal} in a similar fashion to \cite{balamurugan2016stochastic} to the SCAFFOLD-S algorithm.  We ask the SCAFFOLD-S algorithm to solve a series of regularized federated minimax optimization problems, where at meta-iteration $t$ the client losses are regularized using parameter $\theta$ as follows:
\begin{align}
    \label{eq:regularizedclient}
    f_i^{\theta}(x,y) := f_i(x,y) + \frac{\theta}{2} \|x - \bar{x}^t\|^2 - \frac{\theta}{2} \|y - \bar{y}^t\|^2
\end{align}
which implies that SCAFFOLD-S is solving the following minimax optimization problem at meta-iteration $t$:
\begin{align}
    \label{eq:regularizedglobal}
    f^{\theta}(x,y) := f(x,y) + \frac{\theta}{2} \|x - \bar{x}^t\|^2 - \frac{\theta}{2} \|y - \bar{y}^t\|^2
\end{align}
The SCAFFOLD-S algorithm solves problem (\ref{eq:regularizedglobal}) to adequete precision (after a precise number of communication rounds depending on $\theta$), and the meta-iterate $\bar{z}^{t+1} := (\bar{x}^{t+1}, \bar{y}^{t+1})$ is set to the solution that is found.  Notice that applying this regularization increases the smoothness of the problem that SCAFFOLD-S has to solve from to $\beta + \theta$, while increasing the strong convex-concavity to $\mu + \theta$.  This gives us a new condition number of $\frac{\beta + \theta}{\mu + \theta}$ for the subproblem.  When $\mu$ is small, as it often is in practice, applying this regularization can speed up SCAFFOLD-S significantly.

This creates a balance between how many inner communication rounds needed to solve problem (\ref{eq:regularizedglobal}) to proper precision, and how many meta-iterations needed to solve the original problem (\ref{eq:objective}), controlled by $\theta$.  If $\theta$ is small, then SCAFFOLD-S will have a harder time finding a solution to (\ref{eq:regularizedglobal}), but the meta-iterations will converge quickly.  For example, if we simply set $\theta = 0$, then we are solving the original objective, and so we only need one meta-iteration.  If $\theta$ is large, then SCAFFOLD-S can find the solution to (\ref{eq:regularizedglobal}) quickly, but more meta-iterations will be needed.  

Our application of the catalyst framework stands in contrast to other attempts to accelerate federated (minimization) algorithms \cite{yuan2020federatedac} \cite{karimireddy2020mime} which accelerate client iterates.  In fact, it was proven in \cite{yuan2020federatedac} that Nesterov accelerated iterates are not inital-value stable; that is, two different instances of Nesterov acceleration may drift exponentially far apart in a finite number of local steps.  

Using the catalyst framework to accelerate federated learning does not suffer from the same issues.  Let $G_i^{\theta}(z;\bar{z}^t)$ be the gradient mapping for the $i$-th client at the $t$-th meta iterate evaluated at $z$.  
First, observe that the client heterogeneity doesn't change by using catalyst:
\begin{align*}
    & \ \ \ \|G_i^{\theta_t}(z;\bar{z}^t) - G_j^{\theta}(z; \bar{z}^t) \|^2 \\
    &= \|G_i(z) - \theta(z - \bar{z}^t) - G_j(z) + \theta(z-\bar{z}^t)\|^2 \\
    &= \|G_i(z)  - G_j(z) \|^2
\end{align*}
Second, under regularization, client iterates are encouraged to be closer together.  When using SCAFFOLD-S, the gradient mapping direction approximates the global gradient mapping.  So for two clients $i,j$, after one iteration, as long as $\gamma_l \leq \frac{1}{\beta + \theta}$,
\begin{align*}
    & \ \ \ \|z_i^{k+1} - z_j^{k+1}\|^2 \\
    &= \|z_i^{k+1} - \gamma_l g_i^k - z_j^{k+1} + \gamma_l g_j^k\|^2 \\
    &\approx \|z_i^k - \gamma_l G^{\theta}(z_i^k; \bar{z}^t) - z_j^k + \gamma_l G^{\theta}(z_j^k; \bar{z}^t)\|^2 \\
    &\leq (1 - \gamma_l(\theta + \mu)) \|z_i^k - z_j^k\|^2
\end{align*}
Where the last inequality follows from contractivity of the gradient.  Therefore, increasing $\theta$ can actually \textit{reduce} the error from client drift.  In fact, in the limiting case where we take $\theta \to \infty$, all clients are solving the same quadratic minimization problem, and will converge to the same point.

With this intuition and an appropriate choice of $\theta$, we get the following guarantee:
\begin{theorem}
    \label{catalysttheorem}
    Given that $\{f_i\}$ are all $\beta$-smooth and $\mu$-strongly convex for $\mu > 0$, the output of SCAFFOLD-Catalyst-S, $z'$, has $\E \|z' - z^*\|^2 \leq \epsilon$ in
    \begin{align}
        \tilde{O}(\min\{\frac{\beta}{\mu}, \frac{p \beta^2}{\mu^2}  + \frac{p \sigma^2}{n \mu^2 \epsilon}+ \frac{p^{1/2}\beta \sigma}{ \mu^{2} \epsilon^{1/2}}+ \frac{\beta \zeta}{ \mu^{2} \epsilon^{1/2} } \})
    \end{align}
    communication rounds, given appropriate choices of $\theta$ and stepsizes.  
\end{theorem}
\textbf{Proof Sketch.}  First assume that $\sigma = 0$ and that $\zeta$ is arbitrarily large.  Then if we set $\theta = \beta - \mu$, SCAFFOLD-S will find a good solution to the regularized objective (\ref{eq:regularizedglobal}) in $\tilde{O}(\frac{(\beta - (\beta - \mu))^2}{(\mu -(\beta - \mu))^2}) = \tilde{O}(1)$ number of communication rounds.  On the other hand, the meta-iterates are iterates of the proximal point algorithm \cite{rockafellar1970monotone} which will converge in $\tilde{\mathcal{O}}(\frac{\theta + \mu}{\mu}) $ meta-iterations.  Using that $\theta = \beta - \mu$, we get $\tilde{\mathcal{O}}(\frac{\beta}{\mu})$ meta-iterations.  Therefore, under arbitrary heterogeneity, the communication complexity of SCAFFOLD-Catalyst-S is $\tilde{\mathcal{O}}(\frac{\beta}{\mu})$.

In the low client heterogeneity case, (precisely, $\zeta = o(\mu \epsilon^{1/2})$, we can set $\theta = c \mu$ for some constant $c$ (that can include $0$), we get $\tilde{\mathcal{O}}(\frac{\theta + \mu}{\mu}) = \tilde{\mathcal{O}}(1)$ meta-iterations.  On the other hand, SCAFFOLD-S will find a good solution to the regularized objective (\ref{eq:regularizedglobal}) in $\tilde{O}(\frac{p \beta^2}{(c+1)^2\mu^2} + \frac{\beta \zeta}{ (c+1)^2\mu^{2} \epsilon^{1/2}})$ iterations (under low client heterogeneity).
Therefore in this setting, SCAFFOLD-Catalyst-S has a communication complexity of $\tilde{\mathcal{O}}( \frac{p \beta^2}{\mu^2} + \frac{\beta \zeta}{ \mu^{2} \epsilon^{1/2}})$.

We can now compare to the Minibatch Mirror-prox rate in Table~\ref{complexitytable}.  
In the regime of $\sigma=0$, which is where the Minibatch Mirror-prox rate is valid, we can see that SCAFFOLD-Catalyst-S matches it under arbitrary client heterogeneity, and can improve past the Minibatch Mirror-prox rate when client heterogeneity is low.
\begin{remark}
    \normalfont Handling noise and client heterogeneity in the proof of Theorem~\ref{catalysttheorem} requires additional work beyond the proofs of \cite{lin2015universal} or \cite{balamurugan2016stochastic}, which analyze catalyst only on algorithms with linear rates (i.e. exponentially decreasing function value or distance to optimum).  
\end{remark}

\section{Experiments}
\begin{figure*}[t] 
	\begin{minipage}[b]{0.32\linewidth}
		\centering
		\includegraphics[width=\textwidth]{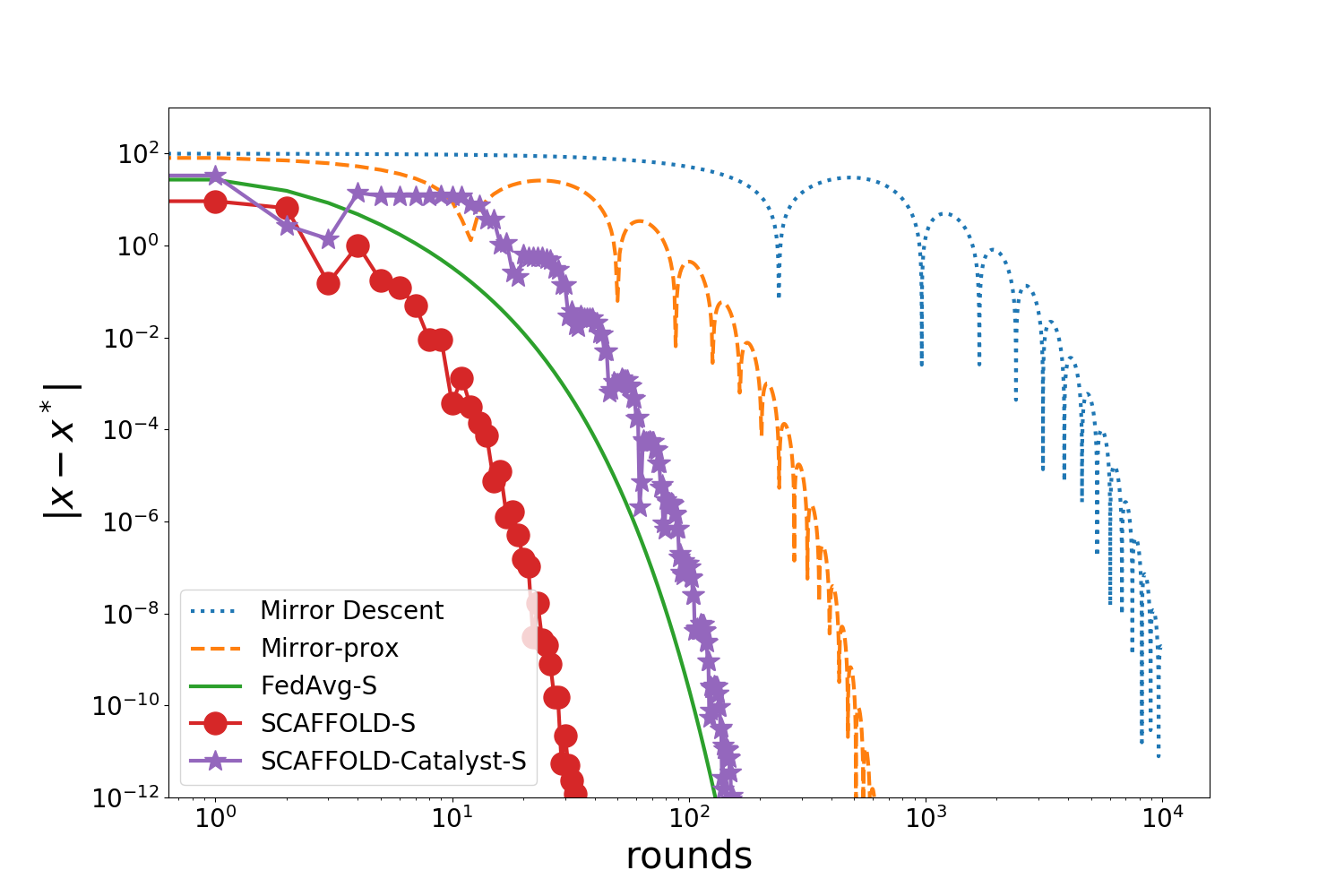}
	\end{minipage}
	~~
	\begin{minipage}[b]{0.32\linewidth}
		\centering
		\includegraphics[width=\textwidth]{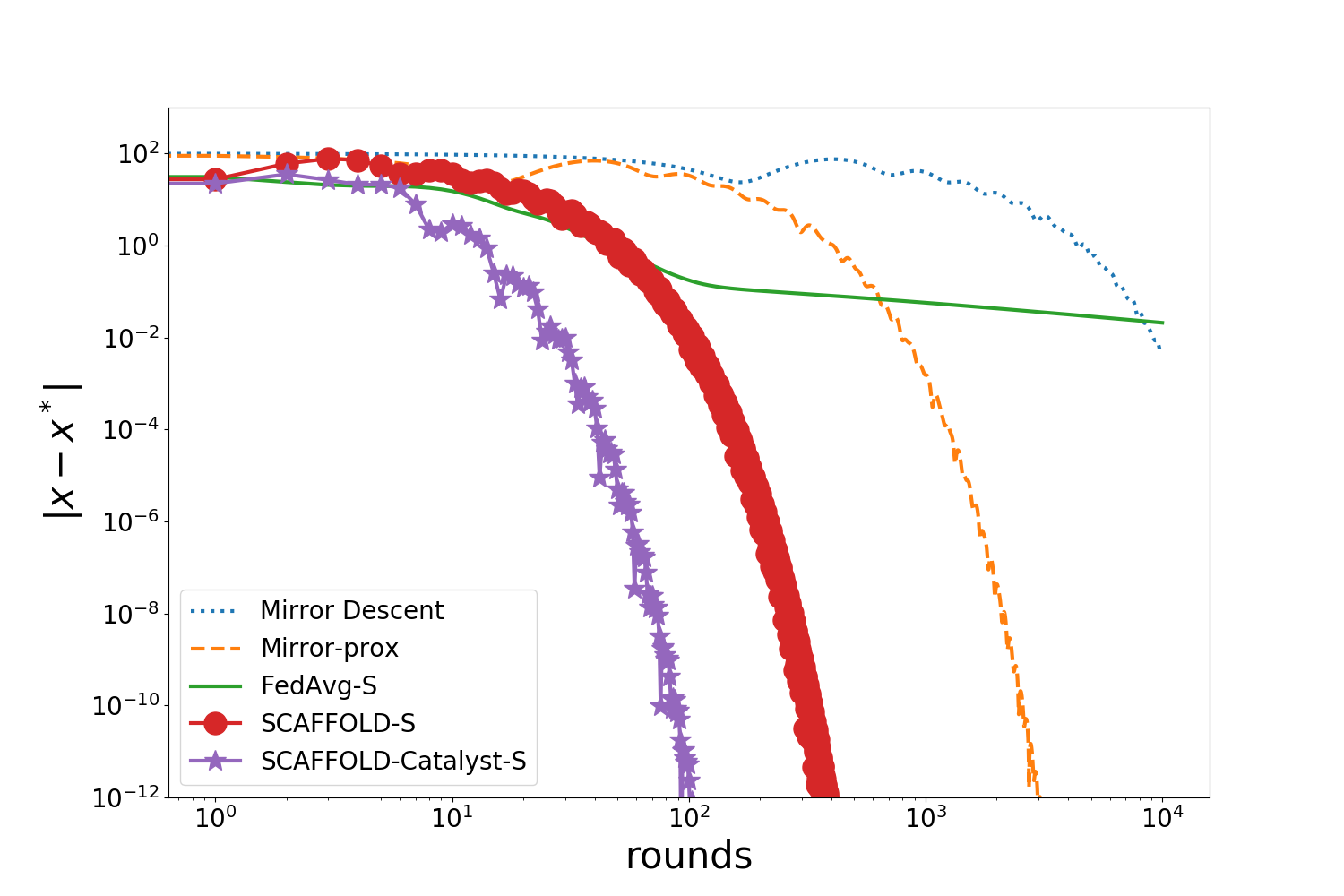}
	\end{minipage}
    ~~
    \begin{minipage}[b]{0.32\linewidth}
        \centering
            \includegraphics[width=\textwidth]{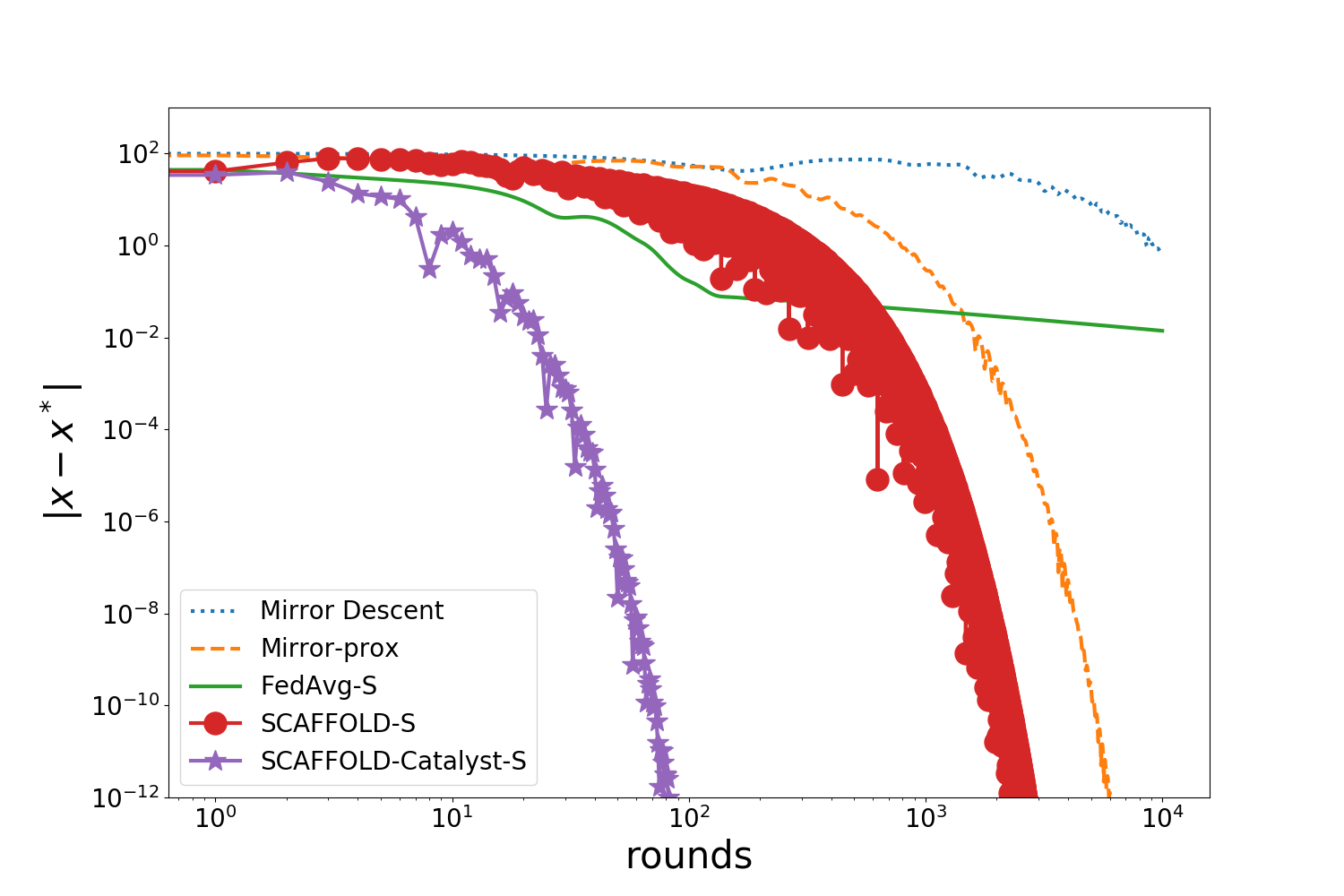}
    \end{minipage}
    \caption{From left to right, all five federated minimax optimization algorithms on $s = 0$ (i.e. clients homogeneous, condition number 1), $s=5$, and $s = 10$.  Each curve is the best curve obtained from the three stepsize choices $\gamma_l = \{0.1, 0.05, 0.01\}/\max(s,1)$.  SCAFFOLD-Catalyst-S maintains its performance even as the client data becomes more heterogeneous and the condition number increases.  }
    \label{fig:trainingcurves}
\end{figure*}
\begin{figure}[t]
    \centering
    \includegraphics[width=0.85\columnwidth]{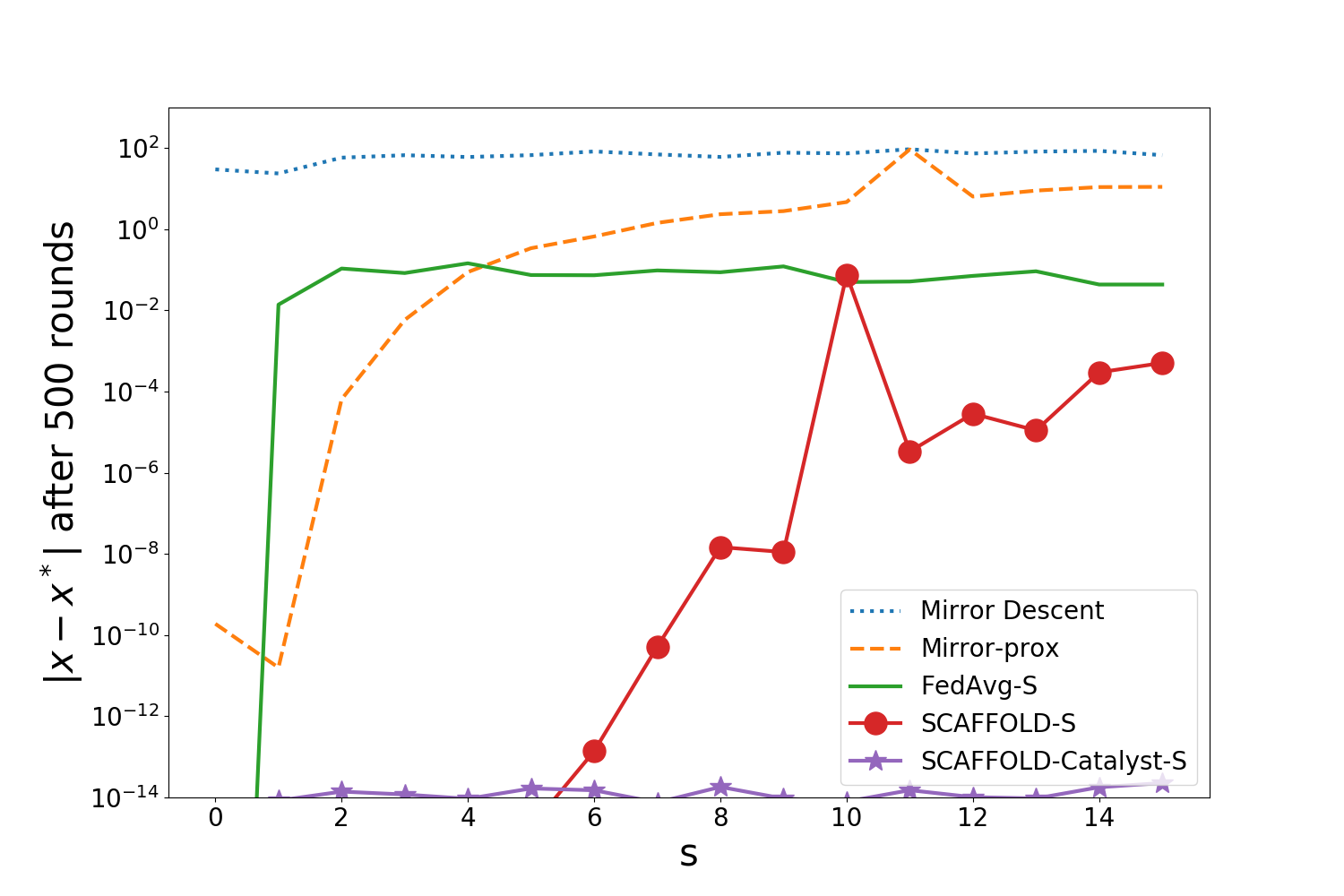}
        \caption{Comparison of the five federated minimax optimization algorithms after 500 communication rounds over $s$.  SCAFFOLD-Catalyst-S maintains its performance over all $s \in [0,15]$, while the other algorithms degrade as client heterogeneity and condition number increase.}
    \label{fig:vshet}
\end{figure}

We run experiments to compare the communication round complexities of Minibatch Mirror Descent (MD), Minibatch Mirror-prox (MP), FedAvg-S, SCAFFOLD-S, and SCAFFOLD-Catalyst-S. 
The experiments show that (i) SCAFFOLD-Catalyst-S is more resilient to client data heterogeneity and ill-conditioning (i.e, larger $\kappa$) than other methods, and (ii) local steps can increase convergence speed with  cross-client variance reduction. 

\paragraph{Setup}
We study the performance of Minibatch MD, Minibach MP, FedAvg-S, SCAFFOLD-S, and SCAFFOLD-Catalyst-S on the following federated minimax problem:
\begin{align}
    \min_{x \in \mathbb{R}^m} \max_{y \in \mathbb{R}^d} -\frac{1}{2} [\frac{1}{n} \sum_{i=1}^n \|y\|^2 - b_i^T y + y^T A_i x] + \frac{\lambda}{2} \|x\|^2
\end{align}
in the language of our original formulation (\ref{eq:objective}),
\begin{align}
    f_i(x,y) := -\frac{1}{2} [ \|y\|^2 - b_i^T y + y^T A_i x] + \frac{\lambda}{2} \|x\|^2
\end{align}
This is a saddle-point formulation of linear regression: 
\begin{align}
    \min_{x \in \mathbb{R}^m} \frac{1}{2} \|\frac{1}{n} \sum_{i=1}^n A_i x - b_i\|^2
\end{align}
with $L_2$ regularization, as shown in \cite{du2019linear}.  This fact was used for experimental evaluation in \cite{mokhtari2020unified}.  We generate our dataset as follows: first, choose a value for parameter $s$: $s$ will control how ill-conditioned the problem is and how heterogeneous the clients are.  Next we sample $b_i' \sim \mathcal{N}(0, s^2 I_d)$, and set $b_i = b_i' - \frac{1}{n} \sum_{i=1}^n b_i'$.  Then we sample $A_i$ by first generating a vector $a_i \sim \mathcal{N}(1^d, s^2 I_d)$ (where $1^d$ is the vector of all ones of dimension $d$), and we threshold $a_i$ so that each entry is at least 1.  Finally, we set $A_i = \text{diag}(a_i)$, where $\text{diag}(v)$ is the diagonalization of $v$ into a matrix.  Here, $d = 10$ and $n = 10$.  One thing to observe is that $\frac{1}{n}\sum_{i=1}^n b_i = 0$, so solution quality is measured by $x$'s distance from the $d$-dimensional zero vector.  We set $\lambda = 0.00001$.

From this, we can see that $s$ controls (1) client data heterogeneity by controlling the heterogeneity in the solutions to each $f_i$, and (2) controls the condition number of $A_i$ by controlling the ratio between the maximum and minimum singular values.

Learning rates for each algorithm were chosen as constant with $\gamma_g = \gamma_l = \{0.1, 0.05, 0.01\}/\max(s,1)$, except for FedAvg-S, which had its stepsizes set as $\gamma_l/(\sqrt{k} + 1)$ (where $k$ is defined as in Section~\ref{sec:baselines}), as we observed empirically that constant stepsize FedAvg-S caused FedAvg-S's convergence to stall out prematurely.  $\tau$ is set deterministically for simplicity as $\tau = 20$.  Furthermore, all gradients are calculated noiselessly, i.e. $\sigma = 0$.  For SCAFFOLD-Catalyst-S, we set $\theta = 1$ in all experiments.

Our training curve results are in Figure~\ref{fig:trainingcurves}.  We condense the curves into Figure~\ref{fig:vshet}, where we can see a clear story: as $s$ increases, all methods other than SCAFFOLD-Catalyst-S degrade.  As predicted by the theory, SCAFFOLD-Catalyst-S is resilient against ill-conditioning and client data heterogeneity.  Furthermore, SCAFFOLD-Catalyst-S can take advantage of local steps, which we can observe from its outperformance of Minibatch Mirror-prox.

\section{Conclusion}
In this work, we have established communication complexity upper bounds for the widely-used federated optimization algorithms FedAvg and SCAFFOLD in the minimax setting.  We find that these complexity bounds are worse than the Minibatch Mirror-prox baseline by a factor of $\kappa$ because both FedAvg and SCAFFOLD are not accelerated algorithms.  We then proposed an accelerated federated minimax optimization algorithm, SCAFFOLD-Catalyst-S, and showed it enjoyed fast convergence properties both in theory and in experiments.  
Future work includes proving tight lower bounds for federated minimax optimization and extending the catalyst acceleration framework to the federated minimization setting.  
Some care will be needed to extend these results 
to the federated minimization setting in full generality (with noise and heterogeneity terms), as the analysis requires the use of Nesterov acceleration machinery.


\nocite{langley00}

\bibliography{ref}
\bibliographystyle{icml2021}

\onecolumn
\appendix
\section{Discussion of \cite{beznosikov2020local}}
While preparing our manuscript, we became aware of \cite{beznosikov2020local}, which is a preliminary work on local step methods for federated minimax optimization.  Their work obtains via their Extra Step Local SGD (disregarding constant factors)
\begin{align}
    \E \|z'- z^*\|^2 \leq \|z^0 - z^*\|^2 \exp(-\mu \gamma K) + \frac{\gamma \sigma^2}{n \mu} + \frac{\tau^3 \gamma^2 \beta^2(\sigma^2 + \zeta^2)}{ \mu^2}
\end{align}
On the other hand, we obtain using FedAvg-S
\begin{align}
    \E \|z' - z^*\|^2 \leq \frac{\beta^2}{\mu^2}\|z^0 - z^*\|^2 \exp(-\mu \gamma K) + \frac{\gamma \sigma^2}{n \mu} + \frac{\gamma^2 \beta^2 \zeta^2}{p^2 \mu^2} + \frac{\gamma^2 \beta^2 \sigma^2}{p \mu^2}
\end{align}
where $z'$ is the output of the algorithms.  That is, they lose a factor of $p$ in the heterogeneity term and two factors of $p$ with respect to the second noise term (setting $p = \frac{1}{\tau}$ for comparison).  Losing factors of $p$ (or $\tau$), makes communication complexity worse, and we can see that our FedAvg-S result strictly improves over their Extra Step Local SGD result (disregarding log factors).

Extra Step Local SGD is just FedAvg-S with the local Mirror Descent steps replaced by local Mirror-prox steps, and Mirror-prox is an accelerated algorithm.  Losing a factor of $p$ in the term bounding client drift shows that they found client drift hard to control when using this accelerated algorithm at the client level.  This suggests that our intuition about client-level acceleration might be true: client-level acceleration in the presence of heterogeneous clients can be harmful to convergence and communication complexity due to how it can amplify client drift.  
\begin{algorithm}[tb]
    \caption{FedAvg-S($\{f_i\}$, p)}
    \label{alg:fedavg}
 \begin{algorithmic}
	\STATE {\bfseries Server Input:} initial $z^0$, probability of communication $p$
	\STATE {\bfseries Server Input:}  Stepsize $\gamma$
    \STATE {\bfseries Client Input:} local function $f_i$
	\STATE {\bfseries set} $\tilde{z}^0 = z^0$
    \FOR{$k=0, 1, \dots$}
		\STATE Flip a coin $c_k$; $c_k = 1$ w.p. $p$, and $c_k = 0$ otherwise
        \FOR{each client $i$ in parallel}
			\STATE $g_i^k \gets \hat{G_i}(z_i^k)$
            \STATE $z_i^{k+1} \gets z_i^{k} - \gamma g_i^k$
            \IF{$c_k = 1$}
				\STATE Clients {\bfseries communicate} $\sum_{l = k'}^k g_i^l$ to server
				\STATE Server {\bfseries broadcasts} $\frac{1}{n}\sum_{i=1}^n \sum_{l = k'}^k g_i^l$ 
				\STATE $\tilde{z}^{k+1} \gets \tilde{z}^k - \gamma_g \frac{1}{n} \sum_{i=1}^n \sum_{l = k'}^k g_i^l$
                \STATE Server {\bfseries broadcasts} $\tilde{z}^{k+1}$
				\STATE $z_i^{k+1} \gets \tilde{z}^{k+1}$
		    \ENDIF
            \
        \ENDFOR
	\ENDFOR
 \end{algorithmic}
 \end{algorithm}
\section{FedAvg-S}
In this section, we let $V_k = \frac{1}{n} \sum_{i=1}^n \|z^k - z_i^k\|^2$, $\sigma_k = \frac{1}{n} \sum_{i=1}^n \| G_i(\tilde{z}^k) - G_i(z^*)\|^2$, where $\tilde{z}^k$ is the last synchronized iterate at or before iteration $k$.  We also use the shorthand $\E_k[\cdot]$ as taking expectation conditioned on everything up to iteration $k$.  We present a detailed description of FedAvg-S in Algorithm~\ref{alg:fedavg}.
\begin{theorem}
    \label{fedavg:proof}
	If we set 
	\begin{align}
		\gamma_l = \gamma_g = 
			\gamma = \min \{\gamma_{\max}, \frac{\log(\max\{2, \min \{\frac{a \mu^2 K^2}{c_1}, \frac{a \mu^3 K^3}{c_2}\} \})}{c_3 \mu K}\}
	\end{align}
	where $a = \frac{\|z^0 - z^*\|^2}{2}$, $c_3 = \frac{1}{4}$, $c_1 = \frac{\sigma^2}{2n}$, $c_2 = \frac{2 \beta^2}{\mu}( \frac{4 c \zeta^2}{p^2} + \frac{2 \sigma^2}{p})$, $\gamma_{\max} = \frac{\mu}{4 \beta^2}$, and let $w_k = (1 - \frac{\gamma \mu}{4})^{1-k}$ such that we return $\frac{1}{W_K} \sum_{k=0}^K w_k z^k$, $c = 9$, and $W_K = \sum_{k=0}^K w_k$, then FedAvg-S an upper bound on expected communication complexity of 
	\begin{align}
		 \tilde{\mathcal{O}}(\frac{p \beta^2}{\mu^2} + \frac{p \sigma^2}{n \mu \epsilon}  + \frac{p^{1/2} \beta \sigma}{\mu^{3/2}\epsilon^{1/2} } + \frac{\beta \zeta}{ \mu^{3/2} \epsilon^{1/2} })
	\end{align}
	with respect to $\gap^*(\cdot)$, and an expected communication complexity of 
	\begin{align}
		\tilde{\mathcal{O}}(\frac{p \beta^2}{\mu^2} + \frac{p \sigma^2}{n \mu^2 \epsilon}  + \frac{p^{1/2} \beta \sigma}{\mu^{2}\epsilon^{1/2} }  + \frac{\beta \zeta}{ \mu^{2} \epsilon^{1/2} })
	\end{align}
	with respect to distance to optimum.
\end{theorem}
\begin{proof}
	From Lemma~\ref{gradientdescent}, we know that 
	\begin{align}
		\gap^*(z^k) \leq \frac{(1 - \gamma \mu + \frac{\gamma \mu}{\alpha})\|z^k - z^*\|^2 - \E_k \|z^{k+1} - z^*\|^2}{2 \gamma} + \frac{\alpha \beta^2}{2 \mu} V_k + \frac{\gamma}{2} \E_k \|g^k\|^2
	\end{align}
	Taking full expectation, using Lemma~\ref{scaffoldlemma}, and setting $\alpha = 2$,
	\begin{align}
		\E \gap^*(z^k) &\leq \frac{(1 - \frac{\gamma\mu}{2})\E \|z^k - z^*\|^2 - \E \|z^{k+1} - z^*\|^2}{2 \gamma} + \frac{ \beta^2}{ \mu} \E V_k \\
		& \ \ \ + \frac{\gamma}{2} [ 2 \beta^2 \E V_k + 2 \beta^2 \E \|z^k - z^*\|^2 + \frac{\sigma^2}{n}] \\
		&= \frac{(1 - \frac{\gamma\mu}{2} + \gamma^2 \beta^2)\E \|z^k - z^*\|^2 - \E \|z^{k+1} - z^*\|^2}{2 \gamma} + (\frac{ \beta^2}{ \mu} + \beta^2 \gamma) \E V_k + \frac{\gamma \sigma^2}{2n}
	\end{align}
	With our stepsize choice $\gamma \leq \frac{1}{\beta} \leq \frac{1}{\mu}$,
	\begin{align}
		\E \gap^*(z^k) &\leq \frac{(1 - \frac{\gamma\mu}{2} + \gamma^2 \beta^2)\E \|z^k - z^*\|^2 - \E \|z^{k+1} - z^*\|^2}{2 \gamma} + \frac{ 2\beta^2}{ \mu} \E V_k + \frac{\gamma \sigma^2}{2n}
	\end{align}
	Using Lemma~\ref{heterodrift},
	\begin{align}
		\E \gap^*(z^k) &\leq \frac{(1 - \frac{\gamma\mu}{2} + \gamma^2 \beta^2)\E \|z^k - z^*\|^2 - \E \|z^{k+1} - z^*\|^2}{2 \gamma} + \frac{ 2\beta^2 \gamma^2}{ \mu} [\frac{4 c  \zeta^2}{p^2} + \frac{2 \sigma^2}{p}] + \frac{\gamma \sigma^2}{2n}
	\end{align}
	If we let $\gamma \leq \frac{\mu}{4\beta^2}$,
	\begin{align}
		\E \gap^*(z^k) &\leq \frac{(1 - \frac{\gamma\mu}{4})\E \|z^k - z^*\|^2 - \E \|z^{k+1} - z^*\|^2}{2 \gamma}  + \frac{ 2\beta^2 \gamma^2}{ \mu} [\frac{4 c  \zeta^2}{p^2} + \frac{2 \sigma^2}{p}] + \frac{\gamma \sigma^2}{2n}
	\end{align}
	Taking a weighted average with $w_k = (1 - \frac{\gamma \mu}{4})^{1 - k}$ and $\sum_{k = 0}^K w_k = W_K$, we get that
	\begin{align}
		\frac{1}{W_K} \sum_{k=0}^K w_k \E \gap^*(z^k) &\leq \frac{\|z^0 - z^*\|^2}{2 \gamma W_K}  + \frac{\gamma \sigma^2}{2n}  + \frac{ 2\beta^2 \gamma^2}{ \mu} [\frac{4 c  \zeta^2}{p^2} + \frac{2 \sigma^2}{p}]
	\end{align}
	Now using Lemma~\ref{linearconvergence} and convex-concavity, we have for $z' = \frac{1}{W_K}\sum_{k=0}^K w_k z^K$,
	\begin{align}
		\E \gap^*(z') =
			\tilde{O}(\frac{a \exp(- c_3 \mu \gamma_{\max} K)}{\gamma_{\max}} + \frac{c_1}{c_3 \mu K} + \frac{c_2}{c_3^2 \mu^2 K^2})
	\end{align}
	where $a = \frac{\|z^0 - z^*\|^2}{2}$, $c_3 = \frac{1}{4}$, $c_1 = \frac{\sigma^2}{2n}$, $c_2 = \frac{2 \beta^2}{\mu}( \frac{4 c \zeta^2}{p^2} + \frac{2 \sigma^2}{p})$, $\gamma_{\max} = \frac{\mu}{4 \beta^2}$, and 
	\begin{align}
		\gamma = \min \{\gamma_{\max}, \frac{\log(\max\{2, \min \{\frac{a \mu^2 K^2}{c_1}, \frac{a \mu^3 K^3}{c_2}\} \})}{c_3 \mu K}\}
	\end{align}
	By strong convex-concavity, we also have that 
	\begin{align}
		\E \|z' - z^*\|^2 =
			\tilde{O}(\frac{a \exp(- c_3 \mu \gamma_{\max} K)}{\mu \gamma_{\max}} + \frac{c_1}{c_3 \mu^2 K} + \frac{c_2}{c_3^2 \mu^3 K^2})
	\end{align}
	By solving both of these convergence rates for $\epsilon$ and multiplying by $p$, we get the communication complexities in the theorem statement.
\end{proof}
\section{SCAFFOLD-S}

In this section, we let $V_k = \frac{1}{n} \sum_{i=1}^n \|z^k - z_i^k\|^2$, $\sigma_k = \frac{1}{n} \sum_{i=1}^n \| G_i(\tilde{z}^k) - G_i(z^*)\|^2$, where $\tilde{z}^k$ is the last synchronized iterate at or before iteration $k$.  We also use the shorthand $\E_k[\cdot]$ as taking expectation conditioned on everything up to iteration $k$.

We have three theorems for SCAFFOLD-S, which correspond to different choices for $\gamma_l$ and $\gamma_g$.  Altogether, they combine to give us the guarantee in Theorem~\ref{scaffoldtheorem}.  

\begin{remark}
    \label{scaffold-s:remark}
    \normalfont Note that we can obtain the statement of Theorem~\ref{scaffoldtheorem} with only Theorem~\ref{scaffold:minibatch} and Theorem~\ref{scaffold:unscaled}.  However, we provide Theorem~\ref{scaffold:scaled} to show that the communication complexity is still quite good in the worst case even under a non-trivial $\gamma_l$ setting.  For comparison with prior work, SCAFFOLD \cite{karimireddy2020scaffold} in fact sets $\gamma_l$ to be extremely low--nearly zero if $\sigma^2$ is large--to get their rate of $\frac{\beta}{\mu} + \frac{\sigma^2}{n \tau \mu^2 \epsilon}$.  By doing this, they avoid incurring a term that is roughly $\frac{\beta^{1/2} \sigma}{\tau^{1/2} \mu^{3/2} \epsilon^{1/2}}$, which is exactly the sort of extra term we incur in our Theorem~\ref{scaffold:scaled} where we do not set $\gamma_l$ extremely small (up to $\sqrt{\kappa}$, which is expected as mentioned in the main paper).
\end{remark}
\subsection{Arbitrary Heterogeneity: Zero Local Stepsize}
\begin{theorem}
	\label{scaffold:minibatch}
	If we set 
	\begin{align}
		\gamma_l = 0, \gamma_g = \frac{p \mu}{4 \beta^2}
	\end{align}
	and $w_k = (1 - \frac{p \mu^2}{4 \beta^2})^{1 - k}$ then SCAFFOLD-S has a communication complexity of 
	\begin{align}
		\tilde{O}(\frac{\beta^2}{\mu^2})
	\end{align}
	wrt both $\gap^*$ and distance to optimum.
\end{theorem}
\begin{proof}

	Let $\tau_k$ be the number of steps since the last synchronized iterate on iterate $k$.  
	\begin{align}
		\E_k \| z^{k+1} - z^*\|^2 &\leq (1 - p)\|z^k - z^*\|^2 + p \E_k \|z^k - \gamma_g \tau_k G(z^k) - z^*\|^2 \\
		&= \|z^k - z^*\|^2 - 2 \gamma_g p \tau_k \langle G(z^k), z^k - z^* \rangle + p \gamma_g^2 \tau_k^2 \| G(z^k) - G(z^*)\|^2 \\
	\end{align}
	Taking full expectation and noting that $p = \frac{1}{\E \tau_k}$,
	\begin{align}
		\E \| z^{k+1} - z^*\|^2 \leq \E \|z^k - z^*\|^2 - 2 \gamma_g \E \langle G(z^k), z^k - z^* \rangle +  \frac{\gamma_g^2  \beta^2}{p} \E \| z^k - z^*\|^2
	\end{align}
	Using strong convex-concavity,
	\begin{align}
		\E \gap^*(z^k) \leq \frac{\E \|z^k - z^*\|^2 - \E \| z^{k+1} - z^*\|^2}{2 \gamma_g} - \frac{\mu}{2} \E \|z^k - z^*\|^2 + \frac{\gamma_g \beta^2}{p} \E \| z^k - z^*\|^2
	\end{align}
	If we take $\gamma_g = \frac{p \mu}{4 \beta^2}$, then we get 
	\begin{align}
		\E \gap^*(z^k) &\leq \frac{\E \|z^k - z^*\|^2 - \E \| z^{k+1} - z^*\|^2}{2 \gamma_g} - \frac{\mu}{4} \E \|z^k - z^*\|^2 \\
		&= \frac{(1 - \frac{\gamma_g \mu}{2})\E \|z^k - z^*\|^2 - \E \| z^{k+1} - z^*\|^2}{2 \gamma_g}
	\end{align}
	So by setting $w_k = (1 - \frac{\gamma_g \mu}{2})^{1-k}$, we get that 
	\begin{align}
		\frac{1}{W_K}\sum_{k=0}^K w_k \E \gap^*(z^k) \leq \frac{\|z^0 - z^*\|^2}{2 \gamma_g} \exp(-\gamma_g \mu K)
	\end{align}
	Using convex-concavity we have that if $z' := \frac{1}{W_K} \sum_{k=0}^K w_k z^k$,
	\begin{align}
		\E \gap^*(z') \leq \frac{\|z^0 - z^*\|^2}{2 \gamma_g} \exp(-\gamma_g \mu K)
	\end{align}
	with our setting of $\gamma_g = \frac{p \mu}{4 \beta^2}$, 
	\begin{align}
		\E \gap^*(z') \leq \frac{2 \beta^2 \|z^0 - z^*\|^2}{p \mu} \exp(-\frac{p \mu^2}{4 \beta^2}K)
	\end{align}
	Which leads to the communication complexity in the theorem statement, by solving for $\epsilon$ and multiplying by $p$.  With respect to distance from optimum, we can again use strong convex-concavity to get 
	\begin{align}
		\E \|z' - z^*\|^2 \leq \frac{4 \beta^2 \|z^0 - z^*\|^2}{p \mu^2} \exp(-\frac{p \mu^2}{4 \beta^2}K)
	\end{align}
	which leads to the same communication complexity.
\end{proof}
\subsection{Arbitrary Heterogeneity: Scaled Stepsizes}
\begin{theorem}
    \label{scaffold:scaled}
	If we set $\gamma := \gamma_l = \gamma_g$, such that 
	\begin{align}
		\gamma = \min \{\gamma_{\max}, \frac{\log(\max\{2, \min \{\frac{a \mu^2 K^2}{c_1}, \frac{a \mu^3 K^3}{c_2}\} \})}{c_3 \mu K}\}
	\end{align} 
	where $a = \frac{\|z^0 - z^*\|^2}{2} + \frac{2 \beta^2 }{\mu} H \gamma^3 \sigma_0^2$, $c_1 = \frac{\sigma^2}{2n}$, $c_2 = \frac{8 \beta^2 \sigma^2}{p \mu}$, $\gamma_{\max} = \frac{p \mu}{80 \beta^2}$, $c_3 = \frac{1}{8}$, $H = \frac{64 (1 - p)(2 + p)(8 + p)}{12 p^3}$, 
	and set $w_k = (1 - \frac{\gamma \mu}{8})^{1-k}$,
	SCAFFOLD-S has an upper bound on expected communication complexity of 
	\begin{align}
		\tilde{\mathcal{O}}(\frac{p \sigma^2}{n \mu \epsilon} + \frac{p^{1/2}\beta \sigma}{ \mu^{3/2} \epsilon^{1/2}} +  \frac{\beta^2}{\mu^2})
	\end{align}
	with respect to $\gap^*(\cdot)$
	and 
	\begin{align}
		\tilde{\mathcal{O}}(\frac{p \sigma^2}{n \mu^2 \epsilon} + \frac{p^{1/2}\beta \sigma}{ \mu^{2} \epsilon^{1/2}} +  \frac{\beta^2}{\mu^2})
	\end{align}
	with respect to distance to optimum.
\end{theorem}
\begin{proof}
	From Lemma~\ref{gradientdescent}, we know that 
	\begin{align}
		\gap^*(z^k) \leq \frac{(1 - \gamma \mu + \frac{\gamma \mu}{\alpha})\|z^k - z^*\|^2 - \E_k \|z^{k+1} - z^*\|^2}{2 \gamma} + \frac{\alpha \beta^2}{2 \mu} V_k + \frac{\gamma}{2} \E_k \|g^k\|^2
	\end{align}
	Taking full expectation and using Lemma~\ref{scaffoldlemma} to bound the last term,
	\begin{align}
		\E \gap^*(z^k) &\leq \frac{(1 - \gamma \mu + \frac{\gamma \mu}{\alpha} +  \gamma^2 \beta^2)\E \|z^k - z^*\|^2 - \E \|z^{k+1} - z^*\|^2}{2 \gamma} + (\frac{\alpha \beta^2}{2 \mu} +  \gamma \beta^2 ) \E V_k   + \frac{\gamma \sigma^2}{2 n}
	\end{align}
	This implies that if we choose $\alpha = 2$, 
	\begin{align}
		\label{plugbackdrift}
		\sum_{k=0}^K w_k \E \gap^*(z^K) &\leq \sum_{k=0}^K w_k [\frac{(1 - \frac{\gamma \mu}{2} +  \gamma^2 \beta^2)\E \|z^k - z^*\|^2 - \E \|z^{k+1} - z^*\|^2}{2 \gamma}] \\
		& \ \ \  + (\frac{ \beta^2}{ \mu} +  \gamma \beta^2) \sum_{k=0}^K w_k \E V_k \\
		& \ \ \ +  \frac{\gamma \sigma^2}{2 n} W_K \\
	\end{align}
	Now using Lemma~\ref{lemma:scaffolddrift}, recall that with $H = \frac{64 (1 - p)(2 + p)(8 + p)}{12 p^3}$ and stepsize restrictions $\gamma \leq \frac{p}{4 \beta\sqrt{(1-p)(2+p)}}$, $\gamma \leq \frac{p}{2\mu}$,
	\begin{align}
		\sum_{k=0}^k w_k \E V_k &\leq \frac{64 (1 - p)(2 + p) \beta^2 \gamma^2 }{3 p^2}(2 + \frac{8}{1 - p}) \sum_{k=0}^K w_k \|z^k - z^*\|^2 + H \gamma^2 \sigma_0^2  + \frac{4 (1-p) \gamma^2 \sigma^2 W_K}{p}
	\end{align}
	We add more stepsize restrictions: we choose $\gamma$ such that 
	\begin{align}
	\frac{64 (1 - p)(2 + p) \beta^4 \gamma^2 }{3 p^2 \mu }(2 + \frac{8}{1 - p}) \leq \frac{\mu}{16} \implies \gamma \leq \frac{p \mu \sqrt{3}}{\beta^2 \sqrt{1024 (1 - p)(2+p) (2 + \frac{8}{1-p})}}
	\end{align} 
	Then we get 
	\begin{align}
		\sum_{k=0}^k w_k \E V_k &\leq \frac{\mu^2}{16\beta^2} \sum_{k=0}^K w_k \|z^k - z^*\|^2 + H \gamma^2  \sigma_0^2  + \frac{4 (1-p) \gamma^2 \sigma^2 W_K}{p}
	\end{align}
	Plugging this back into (\ref{plugbackdrift})
	\begin{align}
		\sum_{k=0}^K w_k \E \gap^*(z^k) &\leq \sum_{k=0}^K w_k [\frac{(1 - \frac{\gamma \mu}{2} +  \gamma^2 \beta^2 + \frac{\gamma \mu}{8} + \frac{\gamma^2 \mu^2}{8})\E \|z^k - z^*\|^2 - \E \|z^{k+1} - z^*\|^2}{2 \gamma}] \\
		& \ \ \ +  \frac{\gamma \sigma^2}{2 n} W_K +  (\frac{ \beta^2}{ \mu} +  \gamma \beta^2) H \gamma^2  \sigma_0^2  + (\frac{ \beta^2}{ \mu} +  \gamma \beta^2) \frac{4 (1-p) \gamma^2 \sigma^2 W_K}{p} \\
	\end{align}
	If we also make sure $\gamma^2 \beta^2 \leq \frac{\gamma \mu}{8} \implies \gamma \leq \frac{\mu}{8 \beta^2}$ and $\gamma \leq \frac{1}{\beta} \leq \frac{1}{\mu}$, then 
	\begin{align}
		\sum_{k=0}^K w_k \E \gap^*(z^k) &\leq \sum_{k=0}^K w_k [\frac{(1 - \frac{\gamma \mu}{8})\E \|z^k - z^*\|^2 - \E \|z^{k+1} - z^*\|^2}{2 \gamma}] \\
		& \ \ \ +  \frac{\gamma \sigma^2}{2 n} W_K +  \frac{ 2 \beta^2}{ \mu} H \gamma^2 \sigma_0^2  + \frac{ 2 \beta^2}{ \mu} \frac{4 (1-p) \gamma^2 \sigma^2 W_K}{p}
	\end{align}
	With the choice of $w_k = (1 - \frac{\gamma \mu}{8})^{1 - k}$, this implies along with convexity and letting the output be $z' = \frac{1}{W_K}\sum_{k=0}^K w_k z^k$ with $\sum_{k = 0}^K w_k = W_K$,
	\begin{align}
		\label{withgamma}
		 \E \gap^*(z') &\leq   [\frac{\|z^0 - z^*\|^2}{2 \gamma} +  \frac{ 2 \beta^2}{ \mu} H  \gamma^2 \sigma_0^2] \exp( - \frac{\gamma \mu K}{8})  +  \frac{\gamma \sigma^2}{2 n}    + \frac{ 8 \beta^2 \sigma^2 \gamma^2}{ p \mu}
	\end{align}
	With the stepsize constraints $\gamma \leq \min \{\frac{1}{\beta}, \frac{\mu}{8 \beta^2}, \frac{p \mu \sqrt{3}}{\beta^2 \sqrt{2048(1 - p)(2+p) + 4096 (2+p)}}, \frac{p}{2\mu},  \frac{p}{4 \beta\sqrt{(1-p)(2+p)}} \}$. Choosing $\gamma \leq \frac{p \mu}{80 \beta^2}$ satisfies all constraints.

	We now apply Lemma~\ref{linearconvergence}.  If we then let $a = \frac{\|z^0 - z^*\|^2}{2} + \frac{2 \beta^2 }{\mu} H \gamma^3 \sigma_0^2$, $c_1 = \frac{\sigma^2}{2n}$, $c_2 = \frac{8 \beta^2 \sigma^2}{p \mu}$, $\gamma_{\max} = \frac{p \mu}{80 \beta^2}$, $c_3 = \frac{1}{8}$ and choose 
	\begin{align}
		\gamma = \min \{\gamma_{\max}, \frac{\log(\max\{2, \min \{\frac{a \mu^2 K^2}{c_1}, \frac{a \mu^3 K^3}{c_2}\} \})}{c_3 \mu K}\}
	\end{align}
	then we get via convex-concavity 
	\begin{align}
		\E \gap^*(z') = \tilde{O}(\frac{a \exp(- c_3 \mu \gamma_{\max} K)}{\gamma_{\max}} + \frac{c_1}{c_3 \mu K} + \frac{c_2}{c_3^2 \mu^2 K^2})
	\end{align}
	and by strong convex-concavity 
	\begin{align}
		\E \|z' - z^*\|^2  = \tilde{O}(\frac{a \exp(- c_3 \mu \gamma_{\max} K)}{\mu \gamma_{\max}} + \frac{c_1}{c_3 \mu^2 K} + \frac{c_2}{c_3^2 \mu^3 K^2})
	\end{align}
	Solving for $\epsilon$ and multiplying by $p$ for both gives the communication complexity bounds given in the theorem.
\end{proof}
\subsubsection{Client Drift Under Scaled Stepsizes}
\begin{lemma}
	\label{lemma:scaffolddrift}
	For SCAFFOLD-S, with stepsize $\gamma_l = \gamma_g = \gamma \leq \frac{p}{4 \beta\sqrt{(1-p)(2+p)}}$, $\gamma \leq \frac{p}{2\mu}$,
	\begin{align}
		\sum_{k=0}^K w_k \E V_k &\leq \frac{64 (1 - p)(2 + p) \beta^2 \gamma^2 }{3 p^2}(2 + \frac{8}{1 - p}) \sum_{k=0}^K w_k \|z^k - z^*\|^2 + H \gamma^2 \E \sigma_0^2  + \frac{4 (1-p) \gamma^2 \sigma^2 W_K}{p}
	\end{align}
\end{lemma}
\begin{proof}
	This proof is similar to that of \cite{gorbunov2020local}.
	\begin{align}
		\E_k V_{k+1} &= \frac{1}{n} \sum_{i=1}^n \E_k \|z_i^{k+1} - z^{k+1}\|^2 \\
		&= \frac{1 - p}{n} \sum_{i=1}^n \E_k \|z_i^k - \gamma g_i^k - z^k + \gamma g^k\|^2 \\
		&= \frac{1 - p}{n} \sum_{i=1}^n \|z_i^k - z^k - \gamma \E_k g_i^k + \gamma \E_k g^k\|^2 + \frac{(1-p) \gamma^2}{n} \sum_{i=1}^n \E_k \|g_i^k - \E_k g_i^k - (g^k - \E_k g^k)\|^2 \\
		&\leq (1-\frac{p}{2}) V_k + \frac{(1 - p)(2 + p) \gamma^2}{pn} \sum_{i=1}^n \|\E_k g_i^k\|^2 + \frac{(1-p) \gamma^2}{n} \sum_{i=1}^n \E_k \|g_i^k - \E_k g_i^k \|^2
	\end{align}
	Taking full expectation together with Lemma~\ref{scaffoldlemma}, and letting $r_k := \E \|z^k - z^*\|^2$,
	\begin{align}
		\E V_{k+1} &\leq (1-\frac{p}{2}) \E V_k + \frac{(1 - p)(2 + p) \gamma^2}{pn} \sum_{i=1}^n \E \|\E_k g_i^k\|^2 + \frac{(1-p) \gamma^2}{n} \sum_{i=1}^n \E \|g_i^k - \E_k g_i^k \|^2 \\
		&\leq (1-\frac{p}{2}) \E V_k + \frac{(1 - p)(2 + p) \gamma^2}{p} [4 \beta^2 \E V_k + 4 \beta^2 r_k + 2 \E \sigma_k] + (1-p) \gamma^2 \sigma^2  \\
	\end{align}
	If we take $\gamma$ s.t. $\frac{(1-p)(2+p) 4 \beta^2 \gamma^2}{p} \leq \frac{p}{4} \implies \gamma \leq \frac{p}{4 \beta\sqrt{(1-p)(2+p)}}$ then
	\begin{align}
		\E V_{k+1} \leq (1 - \frac{p}{4}) \E V_k + \frac{(1 - p)(2 + p) \gamma^2}{p}[4 \beta^2 r_l + 2 \E \sigma_l] + (1-p) \gamma^2 \sigma^2
	\end{align}
	By unrolling we get
	\begin{align}
		\E V_{k+1} \leq \frac{(1 - p)(2 + p) \gamma^2}{p} \sum_{l=0}^k (1 - \frac{p}{4})^{k-l} [4 \beta^2 r_l + 2 \E \sigma_l] + (1-p) \gamma^2 \sigma^2 \sum_{l=0}^k (1 - \frac{p}{4})^{k-l}
	\end{align}
	And so altogether
	\begin{align}
		\label{altogether}
		\sum_{k=0}^K w_k \E V_{k} &\leq \frac{(1 - p)(2 + p) \gamma^2}{(1 - \frac{p}{4})p} \sum_{k=0}^K \sum_{l=0}^k (1 - \frac{p}{4})^{k-l} w_k [4 \beta^2 r_l + 2 \E \sigma_l] \\
		& \ \ \ \frac{(1-p) \gamma^2 \sigma^2}{n} \sum_{k=0}^K \sum_{l=0}^k (1 - \frac{p}{4})^{k-l -1} w_k
	\end{align}

	Observe that if $\gamma \leq \frac{p}{2\mu}$,
	\begin{align}
		\label{wk}
		w_k = (1 - \frac{\gamma \mu}{8})^{-(k-i + 1)} (1 - \frac{\gamma \mu}{8})^{-i} \leq w_{k-1} (1 + \frac{\gamma\mu}{4})^i \leq w_{k-i} (1 + \frac{p}{8})
	\end{align}
	And so 
	\begin{align}
		\sum_{k=0}^K \sum_{l=0}^k (1 - \frac{p}{4})^{k-l} w_k [4 \beta^2 r_l + 2 \E \sigma_l] 
		&\leq \sum_{k=0}^K \sum_{l=0}^k (1 - \frac{p}{4})^{k-l} (1 + \frac{p}{8})^{k-l} w_l [4 \beta^2 r_l + 2 \E \sigma_l] \\
		&\leq \sum_{k=0}^K \sum_{l=0}^k (1 - \frac{p}{8})^{k-l} w_l [4 \beta^2 r_l + 2 \E \sigma_l] \\
		&\leq (\sum_{k=0}^K w_k [4 \beta^2 r_l + 2 \E \sigma_l])(\sum_{k=0}^{\infty} (1 - \frac{p}{8})^k) \\
		\label{finalp}
		&= \frac{8}{p} \sum_{k=0}^K w_k [4 \beta^2 r_l + 2 \E \sigma_l]
	\end{align}
	and also 
	\begin{align}
		\sum_{k=0}^K \sum_{l=0}^k (1 - \frac{p}{4})^{k-l-1} w_k \leq (\sum_{k=0}^K w_k)(\sum_{k=0}^{\infty} (1 - \frac{p}{4})^k) = \frac{4 W_K}{p}
	\end{align}
	Plugging back in to (\ref{altogether}) and using $1 - \frac{p}{4} \geq \frac{3}{4}$,
	\begin{align}
		\label{backtodrift}
		\sum_{k=0}^K w_k \E V_{k} \leq \frac{32 (1 - p)(2 + p) \gamma^2}{3 p^2}\sum_{k=0}^K w_k [4 \beta^2 r_k + 2 \E \sigma_k] + \frac{4 (1-p) \gamma^2 \sigma^2 W_K}{p}  
	\end{align}
	Now observe that from Lemma~\ref{scaffoldlemma},
	\begin{align}
		\E \sigma_{k+1}^2 &\leq (1 - p) \E \sigma_k^2 + p \beta^2 r_k \\
		&\leq (1 - p)^{k+1} \E \sigma_0^2 + \sum_{l=0}^k (1 - p)^{k-l} p \beta^2 r_l
	\end{align}
	And so we have from (\ref{wk}) and using the calculations ending at (\ref{finalp})
	\begin{align}
		\sum_{k=0}^K w_k \E \sigma_k^2 &\leq \E \sigma_0^2 \sum_{k=0}^K w_k (1 - p)^{k} + \frac{p \beta^2}{1-p}\sum_{k=0}^K \sum_{l=0}^k  (1 - p)^{k-l} w_k r_l \\
		&\leq \E \sigma_0^2 \sum_{k=0}^K w_k (1 - p)^{k} + \frac{p \beta^2}{1-p}\sum_{k=0}^K \sum_{l=0}^k (1 - \frac{p}{4})^{k-l} w_k r_l \\
		\label{sigmas}
		&\leq \E \sigma_0^2 \sum_{k=0}^K w_k (1 - p)^{k} + \frac{ 8 p \beta^2}{p(1-p)}\sum_{k=0}^K w_k r_k
	\end{align}
	And using the fact that $w_k \leq (1 + \frac{\gamma \mu}{4})^{k+1} \leq (1 + \frac{p}{8})^{k+1}$,
	\begin{align}
		\sum_{k=0}^K w_k (1 - p)^{k} &\leq (1 + \frac{p}{8}) \sum_{k=0}^K (1 + \frac{p}{8})^k (1 - p)^k \\
		&\leq (1 + \frac{p}{8})\sum_{k=0}^K (1 - \frac{p}{2})^k \\
		&\leq \frac{2(1 + \frac{p}{8})}{p} = \frac{8 + p}{4p}
	\end{align}
	So plugging back into (\ref{sigmas}),
	\begin{align}
		\sum_{k=0}^K w_k \E \sigma_k^2 \leq \frac{\E \sigma_0^2 (8 + p)}{4p} + \frac{ 8  \beta^2}{(1-p)}\sum_{k=0}^K w_k r_k
	\end{align}
	And now plugging back to (\ref{backtodrift}),
	\begin{align}
		\sum_{k=0}^K w_k \E V_{k} &\leq \frac{128 (1 - p)(2 + p)  \beta^2\gamma^2}{3 p^2}\sum_{k=0}^K w_k r_k + \frac{64 (1 - p)(2 + p) \gamma^2}{3 p^2} \E w_k \sigma_k + \frac{4 (1-p) \gamma^2 \sigma^2 W_K}{p} \\
		&\leq \frac{128 (1 - p)(2 + p)  \beta^2\gamma^2}{3 p^2}\sum_{k=0}^K w_k r_k \\
		& \ \ \ + \frac{64 (1 - p)(2 + p) \gamma^2}{3 p^2} [\frac{\E \sigma_0^2 (8 + p)}{4p} + \frac{ 8  \beta^2}{(1-p)}\sum_{k=0}^K w_k r_k] \\
		& \ \ \ + \frac{4 (1-p) \gamma^2 \sigma^2 W_K}{p} \\
		&= \frac{64 (1 - p)(2 + p) \beta^2 \gamma^2 }{3 p^2}(2 + \frac{8}{1 - p}) \sum_{k=0}^K w_k r_k + \frac{64 (1 - p)(2 + p)(8 + p) \gamma^2}{12 p^3} \E \sigma_0^2 \\
		& \ \ \ \ \ \ \ \ + \frac{4 (1-p) \gamma^2 \sigma^2 W_K}{p}
	\end{align}
	Let $H := \frac{64 (1 - p)(2 + p)(8 + p)}{12 p^3} $, and we are done.
\end{proof}
\subsection{Low Heterogeneity: Unscaled Stepsizes}
\begin{theorem}
    \label{scaffold:unscaled}
	If we set 
	\begin{align}
		\gamma_l = \gamma_g = 
			\gamma = \min \{\gamma_{\max}, \frac{\log(\max\{2, \min \{\frac{a \mu^2 K^2}{c_1}, \frac{a \mu^3 K^3}{c_2}\} \})}{c_3 \mu K}\}
	\end{align}
	where $a = \frac{\|z^0 - z^*\|^2}{2}$, $c_3 = \frac{1}{4}$, $c_1 = \frac{\sigma^2}{2n}$, $c_2 = \frac{2 \beta^2}{\mu}( \frac{4 c \zeta^2}{p^2} + \frac{2 \sigma^2}{p})$, $\gamma_{\max} = \frac{\mu}{4 \beta^2}$, and let $w_k = (1 - \frac{\gamma \mu}{4})^{1-k}$ such that we return $\frac{1}{W_K} \sum_{k=0}^K w_k z^k$, $c = 9$, and $W_K = \sum_{k=0}^K w_k$ then SCAFFOLD-S an upper bound on expected communication complexity of 
	\begin{align}
		 \tilde{\mathcal{O}}(\frac{p \beta^2}{\mu^2} + \frac{p \sigma^2}{n \mu \epsilon}  + \frac{p^{1/2} \beta \sigma}{\mu^{3/2}\epsilon^{1/2} } + \frac{\beta \zeta}{ \mu^{3/2} \epsilon^{1/2} })
	\end{align}
	with respect to $\gap^*(\cdot)$, and an expected communication complexity of 
	\begin{align}
		\tilde{\mathcal{O}}(\frac{p \beta^2}{\mu^2} + \frac{p \sigma^2}{n \mu^2 \epsilon}  + \frac{p^{1/2} \beta \sigma}{\mu^{2}\epsilon^{1/2} }  + \frac{\beta \zeta}{ \mu^{2} \epsilon^{1/2} })
	\end{align}
	with respect to distance to optimum.
\end{theorem}
\begin{proof}
    This is the same proof as for Theorem~\ref{fedavg:proof}.
\end{proof}

\section{SCAFFOLD-Catalyst-S}

$\E_t[\cdot]$ be the expectation conditioned on everything up to the $t$-th meta-iteration, and $\sigma_0^2(t) = \frac{1}{n} \sum_{i=1}^n \|G_i^{\theta}(\bar{z}^{t-1}; \bar{z}^{t-1}) - G_i^{\theta}(\prox_{\frac{1}{\theta} f} (\bar{z}^{t-1}); \bar{z}^{t-1})\|^2$.  

We define the following proximal operator for minimax optimization:
\begin{align}
	\prox_g(z') = \arg \min_x \max_y g(x,y) + \frac{1}{2} \|x - x'\|^2 - \frac{1}{2} \|y - y'\|^2
\end{align}
where $z' = (x', y')$.

Observe that if $g$ is $\mu$-strongly convex-concave, then we have the following relation:
\begin{align}
	\|\prox_{\frac{1}{\theta}g}(z_1) - \prox_{\frac{1}{\theta}g}(z_2)\| \leq (1 - \frac{\mu}{\theta + \mu})\|z_1 - z_2\|
\end{align}
This is the strongly convex-concave variant of the ``nonexpansiveness" property of proximal operators \cite{parikh2014proximal}.

\begin{algorithm}[tb]
    \caption{SCAFFOLD-Catalyst-S}
    \label{alg:catalyst2}
 \begin{algorithmic}
    \STATE {\bfseries Server Input:} regularization $\theta$, initial meta-iterate $\bar{z}^0$,
    \STATE probability of communication $p \in (0,1]$, target accuracy $\epsilon$
    \STATE {\bfseries Client Input:} local function $f_i$ 
    \FOR{$t=0, 1, \dots$}
        \STATE {\bfseries communicate} $\bar{z}^t$ to all clients 
        \FOR{each client $i$ in parallel}
            \STATE {\bfseries set} $f_i^{\theta}(z, \xi) = f_i(z, \xi) + \frac{\theta}{2}\|x - \bar{x}^t\|^2 - \frac{\theta}{2}\|y - \bar{y}^t\|^2$
        \ENDFOR
        \STATE $\bar{z}^{t+1} \gets \text{SCAFFOLD-S}(\{f_i^{\theta}\}, p)$ s.t. $\E_t \|\bar{z}^{t+1} - \prox_{\frac{1}{\theta}f}(\bar{z}^t)\|^2 \leq (\frac{\mu}{2(\theta + \mu)})^2 \epsilon$
    \ENDFOR
    \STATE {\bfseries RETURN} $z^T$
 \end{algorithmic}
 \end{algorithm}

 In this section, we will prove the convergence properties of Algorithm~\ref{alg:catalyst2}.  Its counterpart in the main paper was slightly modified for simplicity of presentation.

 Once again, while Theorem~\ref{catalyst:zero} and Theorem~\ref{heterocatalyst} suffice to get the guarantee in Theorem~\ref{catalysttheorem}, we also provide Theorem~\ref{small:theorem}.  This provides a communication complexity guarantee under more realistic settings of stepsizes in the federated setting.  For further discussion on this, see Remark~\ref{scaffold-s:remark}.

\subsection{Arbitrary Heterogeneity: Zero Local Stepsize}
\begin{theorem}
    \label{catalyst:zero}
	By setting $\theta= \beta - \mu$ and running SCAFFOLD-Catalyst-S using the configuration where SCAFFOLD-S is run using the setting of Theorem~\ref{scaffold:minibatch}, the communication complexity of SCAFFOLD-Catalyst-S is 

	\begin{align}
		\tilde{O}(\frac{\beta}{\mu})
	\end{align}
	
\end{theorem}
\begin{proof}
	First, by Minkowski's inequality,
	\begin{align}
		\label{minkowskis:minibatch}
		(\E \| \bar{z}^t - z^*\|^2)^{1/2} \leq (\E\| \bar{z}^t - \prox_{\frac{1}{\theta} f} (\bar{z}^{t-1})\|^2)^{1/2} + (\E \| \prox_{\frac{1}{\theta}f}(\bar{z}^{t-1}) - z^*\|^2)^{1/2}
	\end{align}
	Now observe that by strong convex-concavity, nonexpansiveness of the proximal operator, and the setting of $\theta = \beta - \mu$
	\begin{align}
		\| \prox_{\frac{1}{\theta}f}(\bar{z}^{t-1}) - z^*\|^2 &= \| \prox_{\frac{1}{\theta}f}(\bar{z}^{t-1}) - \prox_{\frac{1}{\theta}f}(z^*) \|^2 \\
		&\leq (1 - \frac{\mu}{\beta})^2  \|\bar{z}^{t-1} - z^*\|^2
	\end{align}
	Notice that the stopping criteria for SCAFFOLD-S is, with the setting of $\theta = \beta - \mu$,
	\begin{align}
		\label{minibatch:stop}
		\E_{t-1} \|\bar{z}^t - \prox_{\frac{1}{\theta} f} (\bar{z}^{t-1})\|^2 \leq (\frac{\mu}{2\beta})^2 \epsilon
	\end{align}
	We can plug back into (\ref{minkowskis:minibatch}):
	\begin{align}
		(\E \| \bar{z}^t - z^*\|^2)^{1/2} \leq \frac{\mu}{2 \beta} \epsilon^{1/2} + (1 - \frac{\mu}{\beta}) \E \|\bar{z}^{t-1} - z^*\|
	\end{align}
	If we unroll the recurrence, we get 
	\begin{align}
		(\E \| \bar{z}^T - z^*\|^2)^{1/2} \leq  (1 - \frac{\mu}{\beta})^T \| \bar{z}^0 - z^*\| + \epsilon^{1/2}
	\end{align}
	By setting $T = \frac{\beta}{\mu} \log(\frac{2 \|\bar{z}^0 - z^*\|}{\epsilon^{1/2}})$
	Then we have that 
	\begin{align}
		\E \| \bar{z}^T - z^*\|^2 \leq \epsilon
	\end{align}
	Now returning to (\ref{minibatch:stop}), we can see from Theorem~\ref{scaffold:minibatch} that this would take $\tilde{O}(1)$ communication rounds for each meta-iteration.  Therefore, the total communication complexity is $\tilde{O}(\frac{\beta}{\mu})$.

\end{proof}

\subsection{Arbitrary Heterogeneity: Scaled Stepsizes}
\begin{theorem}
	\label{small:theorem}
	By setting $\theta = \beta - \mu$ and running SCAFFOLD-Catalyst-S using the configuration where SCAFFOLD-S is run using the setting of Theorem~\ref{scaffold:scaled}, the communication complexity of SCAFFOLD-Catalyst-S is 
	\begin{align}
		 \tilde{O}(\frac{p \beta \sigma^2}{n \mu^3 \epsilon} + \frac{p^{1/2} \beta \sigma}{ \mu^2 \epsilon^{1/2}} +  \frac{\beta}{\mu})
	\end{align}
\end{theorem}
\begin{proof}
	First, by Minkowski's inequality,
	\begin{align}
		\label{minkowskis:scaled}
		(\E \| \bar{z}^t - z^*\|^2)^{1/2} \leq (\E\| \bar{z}^t - \prox_{\frac{1}{\theta} f} (\bar{z}^{t-1})\|^2)^{1/2} + (\E \| \prox_{\frac{1}{\theta}f}(\bar{z}^{t-1}) - z^*\|^2)^{1/2}
	\end{align}
	Now observe that by strong convex-concavity, nonexpansiveness of the proximal operator, and the setting of $\theta = \beta - \mu$
	\begin{align}
		\| \prox_{\frac{1}{\theta}f}(\bar{z}^{t-1}) - z^*\|^2 &= \| \prox_{\frac{1}{\theta}f}(\bar{z}^{t-1}) - \prox_{\frac{1}{\theta}f}(z^*) \|^2 \\
		&\leq (1 - \frac{\mu}{\beta})^2  \|\bar{z}^{t-1} - z^*\|^2
	\end{align}
	Notice that the stopping criteria for SCAFFOLD-S is, with the setting of $\theta = \beta - \mu$,
	\begin{align}
		\label{scaled:stop}
		\E_{t-1} \|\bar{z}^t - \prox_{\frac{1}{\theta} f} (\bar{z}^{t-1})\|^2 \leq (\frac{\mu}{2\beta})^2 \epsilon
	\end{align}
	We can plug back into (\ref{minkowskis:scaled}):
	\begin{align}
		(\E \| \bar{z}^t - z^*\|^2)^{1/2} \leq \frac{\mu}{2 \beta} \epsilon^{1/2} + (1 - \frac{\mu}{\beta}) \E \|\bar{z}^{t-1} - z^*\|
	\end{align}
	If we unroll the recurrence, we get 
	\begin{align}
		(\E \| \bar{z}^T - z^*\|^2)^{1/2} \leq  (1 - \frac{\mu}{\beta})^T \| \bar{z}^0 - z^*\| + \epsilon^{1/2}
	\end{align}
	By setting $T = \frac{\beta}{\mu} \log(\frac{2 \|\bar{z}^0 - z^*\|}{\epsilon^{1/2}})$
	Then we have that 
	\begin{align}
		\E \| \bar{z}^T - z^*\|^2 \leq \epsilon
	\end{align}
	Now returning to (\ref{scaled:stop}), we can see from Theorem~\ref{scaffold:scaled} that this would take $\tilde{\mathcal{O}}(\frac{p \sigma^2}{n \mu^2 \epsilon} + \frac{p^{1/2} \sigma}{ \mu \epsilon^{1/2}} +  1)$ communication rounds for each meta-iteration.  Therefore, the total communication complexity is $\tilde{O}(\frac{p \beta \sigma^2}{n \mu^3 \epsilon} + \frac{p^{1/2} \beta \sigma}{ \mu^2 \epsilon^{1/2}} +  \frac{\beta}{\mu})$.

\end{proof}

\subsection{Low Heterogeneity: Unscaled Stepsizes}
\begin{theorem}
	\label{heterocatalyst}
	By setting $\theta = c \mu$ and running SCAFFOLD-Catalyst-S using the configuration where SCAFFOLD-S is run using the setting of Theorem~\ref{scaffold:unscaled}, the communication complexity of SCAFFOLD-Catalyst-S is 
	\begin{align}
		\tilde{\mathcal{O}}(\lceil 1 + c  \rceil \lceil \frac{p (\beta + c \mu)^2}{(1 + c)^2\mu^2} + \frac{p \sigma^2}{n \mu^2 \epsilon}  + \frac{p^{1/2} (\beta + c \mu) \sigma}{(1 + c)\mu^{2}\epsilon^{1/2} }  + \frac{(\beta + c \mu) \zeta}{ (1 + c)\mu^{2} \epsilon^{1/2} } \rceil )
	\end{align}
\end{theorem}
\begin{proof}
	First, by Minkowski's inequality,
	\begin{align}
		\label{minkowskis:unscaled}
		(\E \| \bar{z}^t - z^*\|^2)^{1/2} \leq (\E\| \bar{z}^t - \prox_{\frac{1}{\theta} f} (\bar{z}^{t-1})\|^2)^{1/2} + (\E \| \prox_{\frac{1}{\theta}f}(\bar{z}^{t-1}) - z^*\|^2)^{1/2}
	\end{align}
	Now observe that by strong convex-concavity, nonexpansiveness of the proximal operator, and the setting of $\theta = c \mu$
	\begin{align}
		\| \prox_{\frac{1}{\theta}f}(\bar{z}^{t-1}) - z^*\|^2 &= \| \prox_{\frac{1}{\theta}f}(\bar{z}^{t-1}) - \prox_{\frac{1}{\theta}f}(z^*) \|^2 \\
		&\leq (1 - \frac{1}{1 + c})^2  \|\bar{z}^{t-1} - z^*\|^2
	\end{align}
	Notice that the stopping criteria for SCAFFOLD-S is, with the setting of $\theta = c\mu$,
	\begin{align}
		\label{unscaled:stop}
		\E_{t-1} \|\bar{z}^t - \prox_{\frac{1}{\theta} f} (\bar{z}^{t-1})\|^2 \leq (\frac{1}{2(1 + c)})^2 \epsilon
	\end{align}
	We can plug back into (\ref{minkowskis:unscaled}):
	\begin{align}
		(\E \| \bar{z}^t - z^*\|^2)^{1/2} \leq \frac{1}{2(1 + c)} \epsilon^{1/2} + (1 - \frac{\mu}{\beta}) \E \|\bar{z}^{t-1} - z^*\|
	\end{align}
	If we unroll the recurrence, we get 
	\begin{align}
		(\E \| \bar{z}^T - z^*\|^2)^{1/2} \leq  (1 - \frac{1}{1 + c})^T \| \bar{z}^0 - z^*\| + \frac{1}{2}\epsilon^{1/2}
	\end{align}
	By setting $T = (c + 1) \log(\frac{2 \|\bar{z}^0 - z^*\|}{\epsilon^{1/2}})$
	Then we have that 
	\begin{align}
		\E \| \bar{z}^T - z^*\|^2 \leq \epsilon
	\end{align}
	Now returning to (\ref{unscaled:stop}), we can see from Theorem~\ref{scaffold:scaled} that this would take $\tilde{\mathcal{O}}(\frac{p (\beta + c \mu)^2}{(1 + c)^2\mu^2} + \frac{p \sigma^2}{n \mu^2 \epsilon}  + \frac{p^{1/2} (\beta + c \mu) \sigma}{(1 + c)\mu^{2}\epsilon^{1/2} }  + \frac{(\beta + c \mu) \zeta}{ (1 + c)\mu^{2} \epsilon^{1/2} })$ communication rounds for each meta-iteration.  Therefore, the total communication complexity is 
	\begin{align}
		\tilde{\mathcal{O}}(\lceil c + 1  \rceil \lceil \frac{p (\beta + c \mu)^2}{(1 + c)^2\mu^2} + \frac{p \sigma^2}{n \mu^2 \epsilon}  + \frac{p^{1/2} (\beta + c \mu) \sigma}{(1 + c)\mu^{2}\epsilon^{1/2} }  + \frac{(\beta + c \mu) \zeta}{ (1 + c)\mu^{2} \epsilon^{1/2} } \rceil )
	\end{align}

\end{proof}
\section{Technical Lemmas}
\subsection{Mirror Descent Lemma}
\begin{lemma}
	\label{gradientdescent}
	For SCAFFOLD-S or FedAvg-S and any $\alpha > 0$,
	\begin{align}
		\gap^*(z^k) \leq \frac{(1 - \gamma \mu + \frac{\gamma \mu}{\alpha})\|z^k - z^*\|^2 - \E_k \|z^{k+1} - z^*\|^2}{2 \gamma} + \frac{\alpha \beta^2}{2 \mu} V_k + \frac{\gamma}{2} \E_k \|g^k\|^2
	\end{align}
\end{lemma}
\begin{proof}
	Let $g^k = \frac{1}{n} \sum_{i=1}^n g_i^k$, and $g_i^k$ be either a gradient or a cross-client variance-reduced gradient that client $i$ uses at the $k$-th step.  Note that $g^k = \frac{1}{n} \sum_{i=1}^n G_i(z_i^k)$ in either case.  Using the definition of the algorithm,
	\begin{align}
		\E_k \|z^{k+1} - z^*\|^2 = \|z^k - z^*\|^2 - 2 \gamma \langle g^k, z^k - z^* \rangle + \gamma^2 \E_k \|g^k\|^2
	\end{align}
	We bound the middle term: 
	\begin{align}
		- 2 \gamma \langle g^k, z^k - z^* \rangle &= -2 \gamma \langle G(z^k), z^k - z^* \rangle + 2 \gamma \langle G(z^k) - g^k, z^k - z^* \rangle \\
	\end{align}
	By strong convexity,
	\begin{align}
		-2 \gamma \langle G(z^k), z^k - z^* \rangle \leq -2 \gamma \gap^*(z^k) - \mu \gamma \|z^k - z^*\|^2
	\end{align}
	And by Fenchel-Young, for some constant $\alpha$,
	\begin{align}
		2 \gamma \langle G(z^k) - g^k, z^k - z^* \rangle &\leq \frac{\alpha \gamma}{\mu} \|G(z^k) - g^k\|^2 + \frac{\gamma \mu}{\alpha}\|z^k - z^*\|^2 \\
		&\leq \frac{\alpha \gamma}{\mu} \|G(z^k) - g^k\|^2 + \frac{\gamma \mu}{\alpha}\|z^k - z^*\|^2 \\
		&= \frac{\alpha \gamma}{\mu} \|\frac{1}{n} \sum_{i=1}^n G_i(z^k) - G_i(z_i^k)\|^2 + \frac{\gamma \mu}{\alpha}\|z^k - z^*\|^2 \\
		&\leq \frac{\alpha \gamma \beta^2 }{\mu} V_k + \frac{\gamma \mu}{\alpha}\|z^k - z^*\|^2
	\end{align}
	Where $V_k := \frac{1}{n} \sum_{i=1}^n \|z_i^k - z^k\|^2$ (client drift).
	
	So altogether, we have that
	\begin{align}
		\gap^*(z^k) \leq \frac{(1 - \gamma \mu + \frac{\gamma \mu}{\alpha})\|z^k - z^*\|^2 - \E_k \|z^{k+1} - z^*\|^2}{2 \gamma} + \frac{\alpha \beta^2}{2 \mu} V_k + \frac{\gamma}{2} \E_k \|g^k\|^2
	\end{align}
\end{proof}

\subsection{Client Drift: Large Local Stepsize}
\begin{lemma}
	\label{heterodrift}
	FedAvg-S and SCAFFOLD-S's client drift satisfies for a universal constant $c \leq 9$
	\begin{align}
		\E V_k \leq \frac{4 c \gamma^2 \zeta^2}{p^2} + \frac{2 \gamma^2 \sigma^2}{p}
	\end{align}
\end{lemma}
\begin{proof}
	This proof is similar to that of \cite{woodworth2020minibatch}.  Given that $\gamma \leq \frac{1}{\beta}$,
	\begin{align}
		& \ \ \E \|z^{k+1}_i - z^{k+1}_j\|^2 \\
		&\leq (1 - p) \E \|z^{k}_i - \gamma G(z^k_i) - z^{k}_j + \gamma G(z^k_j) - \gamma (g_i - G(z^k_i)) + \gamma(g_j - G(z^k_j))\|^2 + \gamma^2 \sigma^2 \\
		&\leq (1 - p)(1 + \frac{p}{2}) \E \|z^{k}_i - \gamma G(z^k_i) - z^{k}_j + \gamma G(z^k_j)\|^2 + \gamma^2 \sigma^2\\
		& \ \ \ \ + (1-p)(1 + \frac{2}{p}) \E \|\gamma (g_i - G(z^k_i)) + \gamma(g_j - G(z^k_j))\|^2 + \gamma^2 \sigma^2\\
		&\leq (1 - \frac{p}{2}) \E \|z^{k}_i - \gamma G(z^k_i) - z^{k}_j + \gamma G(z^k_j)\|^2 \\
		& \ \ \ \ + \frac{2}{p} \E \|\gamma (g_i - G(z^k_i)) + \gamma(g_j - G(z^k_j))\|^2 + \gamma^2 \sigma^2
	\end{align}
	In the FedAvg-S case, we have that
	\begin{align}
		\E \|\gamma (g_i - G(z^k_i)) + \gamma(g_j - G(z^k_j))\|^2 &= \E \|\gamma (G_i(z^k_i) - G(z^k_i)) + \gamma(G_j(z^k_j) - G(z^k_j))\|^2 \\
		&\leq 2 \zeta^2 + 2 \zeta^2 = 4 \gamma^2 \zeta^2
	\end{align}
	In the SCAFFOLD-S case, we have that (where $\tilde{z}^k$ is the reference iterate at the k-th step)
	\begin{align}
		& \ \ \ \E \|\gamma (g_i - G(z^k_i)) + \gamma (g_j - G(z^k_j))\|^2 \\
		&= \E \|\gamma (G_i(z^k_i) - G(z^k_i)) + \gamma(G_j(z^k_j) - G(z^k_j)) + \gamma(G_i(\tilde{z}^k) - G_j(\tilde{z}^k))\|^2 \\
		& \leq 9 \gamma^2 \zeta^2
	\end{align}
	Since the two cases only differ by a constant, from this point on in the proof we will proceed as $\E \|\gamma (g_i - G(z^k_i)) + \gamma(g_j - G(z^k_j))\|^2 \leq c \gamma^2 \zeta^2$, where $c$ is either $4$ or $9$. 
	By unrolling, we see that
	\begin{align}
		\E \|z_i^k - z_j^k\|^2 \leq \frac{2}{p}( \frac{2 c \gamma^2 \zeta^2}{p} + \gamma^2 \sigma^2)
	\end{align}
	By convexity, we therefore have
	\begin{align}
		\E V_k \leq \frac{4 c \gamma^2 \zeta^2}{p^2} + \frac{2 \gamma^2 \sigma^2}{p}
	\end{align}
\end{proof}

\subsection{Gradient Mapping Bounds}
\begin{lemma}
	\label{scaffoldlemma}
	For FedAvg-S and SCAFFOLD-S
	\begin{align}
		\E \| g^k \|^2 \leq 2 \beta^2 \E V_k + 2 \beta^2 \E \|z^k - z^*\|^2 + \frac{\sigma^2}{n}
	\end{align}
	For SCAFFOLD-S, we have that, if we let $\sigma_k^2 := \frac{1}{n} \|G_i(\tilde{z}^k) - G_i(z^*)\|^2$,
	\begin{align}
		\frac{1}{n} \sum_{i=1}^n  \E \|\E_k g_i^k\|^2 \leq 4 \beta^2 \E V_k + 4 \beta^2 \E \|z^k - z^*\|^2 + 2 \E \sigma_k
	\end{align}
	\begin{align}
		\frac{1}{n} \sum_{i=1}^n \E \|g_i^k - \E_k g_i^k\|^2 \leq \sigma^2
	\end{align}
	\begin{align}
		\E \sigma_{k+1}^2 \leq (1 - p) \E \sigma_k^2 + p \beta^2 \E \|z^k - z^*\|^2
	\end{align}
\end{lemma}
\begin{proof}
	\begin{align}
		\E \| g^k \|^2 &\leq \E \| \frac{1}{n} \sum_{i=1}^n G_i(z_i^k)\|^2 + \frac{\sigma^2}{n} \\
		&= \E \| \frac{1}{n} \sum_{i=1}^n G_i(z_i^k) - G_i(z^k) + G_i(z^k) - G_i(z^*)\|^2 + \frac{\sigma^2}{n} \\
		&\leq 2 \beta^2 \E V_k + 2 \beta^2 \E \|z^k - z^*\|^2 + \frac{\sigma^2}{n}
	\end{align}
	\begin{align}
		\frac{1}{n} \sum_{i=1}^n \E \|\E_k g_i^k\|^2 &\leq \frac{1}{n} \sum_{i=1}^n \E \|G_i(z_i^k) - G_i(\tilde{z}^k) + G(\tilde{z}^k)\|^2 \\
		&= \frac{1}{n} \sum_{i=1}^n \E \|G_i(z_i^k) - G_i(z^*) + G_i(z^*) - G_i(\tilde{z}^k) + G(\tilde{z}^k) - G(z^*)\|^2  \\
		&= \frac{2}{n} \sum_{i=1}^n \E \|G_i(z_i^k) - G_i(z^*)\|^2 \\
		& \ \ \ \ + \frac{2}{n} \sum_{i=1}^n \E \|G_i(\tilde{z}^k) - G_i(z^*) - [G(\tilde{z}^k) - G(z^*)]\|^2\\
		&= \frac{2}{n} \sum_{i=1}^n \E \|G_i(z_i^k)- G_i(z^k) + G_i(z^k) - G_i(z^*)\|^2 \\
		& \ \ \ \ + \frac{2}{n} \sum_{i=1}^n \E \|G_i(\tilde{z}^k) - G_i(z^*) - [G(\tilde{z}^k) - G(z^*)]\|^2\\
		&\leq 4 \beta^2 \E V_k + 4 \beta^2 \E \|z^k - z^*\|^2 + 2 \E \sigma_k^2 \\
	\end{align}
	The third follows by definition of $\sigma^2$.
	\begin{align}
		\E_k \sigma_{k+1}^2 &= (1 - p) \sigma_k^2 + p \frac{1}{n} \sum_{i=1}^n \|G_i(z^k) - G(z^*)\|^2 \\
		&\leq (1 - p) \sigma_k^2 + p \beta^2 \|z^k - z^*\|^2
	\end{align}
\end{proof}
\subsection{Linear Convergence Rate}
The following lemma mostly follows \cite{gorbunov2020local}'s Lemma I.2.
\begin{lemma}
	\label{linearconvergence}
	Let a sequence $\{r_k\}_{k \geq 0}$ satisfy
	\begin{align}
		r_K \leq \frac{a}{\gamma W_K} + c_1 \gamma + c_2 \gamma^2
	\end{align}
	where $W_K = \sum_{k=0}^K w_k$, for some $w_k = (1 - c_3 \gamma \mu)^{1-k}$, $c_3 \leq 1$, and with $\gamma \leq \gamma_{\max}$.  Then if we choose
	\begin{align}
		\gamma = \min \{\gamma_{\max}, \frac{\log(\max\{2, \min \{\frac{a \mu^2 K^2}{c_1}, \frac{a \mu^3 K^3}{c_2}\} \})}{c_3 \mu K}\}
	\end{align}
	we will have 
	\begin{align}
		\tilde{O}(\frac{a \exp(- c_3 \mu \gamma_{\max} K)}{\gamma_{\max}} + \frac{c_1}{c_3 \mu K} + \frac{c_2}{c_3^2 \mu^2 K^2})
	\end{align}
\end{lemma}
\begin{proof}
	First observe that $W_K \geq w_K \geq (1 - c_3 \gamma \mu)^{-K}$.  Therefore
	\begin{align}
		r_K \leq \frac{a}{\gamma} \exp(-c_3 \mu \gamma K) + c_1 \gamma + c_2 \gamma^2
	\end{align}
	The first case is if $\gamma_{\max} \leq \frac{\log(\max\{2, \min \{\frac{a \mu^2 K^2}{c_1}, \frac{a \mu^3 K^3}{c_2}\} \})}{c_3 \mu K} $.  If this is the case, we set $\gamma = \gamma_{\max}$, which gives us
	\begin{align}
		& \ \ \ r_K \\
		&\leq \frac{a}{\gamma} \exp(-c_3 \mu \gamma K) + c_1 \gamma + c_2 \gamma^2 \\
		&\leq \frac{a}{\gamma_{\max}} \exp( - c_3 \mu \gamma_{\max} K) + c_1 \gamma_{\max} + c_2 \gamma_{\max}^2 \\
		&\leq \frac{a}{\gamma_{\max}} \exp( - c_3 \mu \gamma_{\max} K) +  \frac{c_1 \log(\max\{2, \min \{\frac{a \mu^2 K^2}{c_1}, \frac{a \mu^3 K^3}{c_2}\} \})}{c_3 \mu K} + \frac{c_2 \log^2(\max\{2, \min \{\frac{a \mu^2 K^2}{c_1}, \frac{a \mu^3 K^3}{c_2}\} \})}{c_3^2 \mu^2 K^2} \\
		&= \tilde{O}(\frac{a}{\gamma_{\max}} \exp( - c_3 \mu \gamma_{\max} K) + \frac{c_1}{c_3 \mu K} + \frac{c_2}{c_3^2 \mu^2 K^2})
	\end{align}
	The other case is if $\gamma_{\max} \geq \frac{\log(\max\{2, \min \{\frac{a \mu^2 K^2}{c_1}, \frac{a \mu^3 K^3}{c_2}\} \})}{c_3 \mu K} $.  
	
	If this is the case, we set 
	$\gamma = \frac{\log(\max\{2, \min \{\frac{a \mu^2 K^2}{c_1}, \frac{a \mu^3 K^3}{c_2}\} \})}{c_3 \mu K}$, which gives us 
	\begin{align}
		& \ \ \ r_K \\
		&\leq \frac{a}{\gamma} \exp(-c_3 \mu \gamma K) + c_1 \gamma + c_2 \gamma^2 \\
		&\leq \frac{c_3 a \mu K}{\max\{2, \min \{\frac{a \mu^2 K^2}{c_1}, \frac{a \mu^3 K^3}{c_2}\} \} \log(\max\{2, \min \{\frac{a \mu^2 K^2}{c_1}, \frac{a \mu^3 K^3}{c_2}\} \})} \\
		& \ \ \ + \frac{c_1 \log(\max\{2, \min \{\frac{a \mu^2 K^2}{c_1}, \frac{a \mu^3 K^3}{c_2}\} \})}{c_3 \mu K} \\
		& \ \ \ + \frac{c_2 \log^2(\max\{2, \min \{\frac{a \mu^2 K^2}{c_1}, \frac{a \mu^3 K^3}{c_2}\} \})}{c_3^2 \mu^2 K^2} \\
		&= \tilde{O}(\frac{c_1}{c_3 \mu K} + \frac{c_2}{c_3^2\mu^2 K^2} + \frac{c_2}{c_3 \mu^2 K^2})
	\end{align}
\end{proof}
\section{Experimental Details}
In our experiments, we move on to the next meta-iteration given sufficient objective decrease as a heuristic, as the requirements on $K_t$ are often too conservative.  The same heuristic was used in \cite{balamurugan2016stochastic} for their experiments.  More principled stopping criterion can be found in \cite{yang2020catalyst}.  Incorporating more convenient stopping criterion for our algorithm is a direction for future work.

Each result for a setting of $s$ took around 5 minutes to run on a 2015 Macbook Pro, though for all our settings of $s$ (0 to 15, counting up by 1) we ran them in parallel on a cluster with 15 Dell Optiplex nodes.

\end{document}